\documentclass{article}
\PassOptionsToPackage{numbers}{natbib}

 \usepackage[final]{neurips_2025}

\usepackage[utf8]{inputenc} %
\usepackage[T1]{fontenc}    %
\usepackage{hyperref}       %
\usepackage{url}            %
\usepackage{booktabs}       %
\usepackage{amsfonts}       %
\usepackage{nicefrac}       %
\usepackage{microtype}      %
\usepackage{graphicx}
\usepackage[table]{xcolor}     %
\usepackage{amsthm}
\usepackage{bm}
\usepackage{bbm}
\usepackage{mathtools}
\usepackage{subcaption}
\usepackage{booktabs}
\usepackage{enumitem}
\usepackage{algorithm}
\usepackage{algorithmic}
\usepackage{siunitx}
\usepackage{adjustbox}
\usepackage{tablefootnote}

\definecolor{gold}{rgb}{1.0, 0.84, 0.0}
\definecolor{silver}{rgb}{0.75, 0.75, 0.75}
\definecolor{bronze}{rgb}{0.8, 0.5, 0.2}

\newcommand{\E}{\mathbb{E}}

\newcommand{\Var}{\textnormal{Var}}

\DeclareMathOperator*{\argmin}{arg\,min}

\usepackage{thm-restate}

  \usepackage{nth}
  \usepackage{intcalc}

  \newcommand{\cAAAI}[1]{AAAI\ Conference\ on\ Artificial (AAAI)}

\title{Regression-adjusted Monte Carlo Estimators for Shapley Values and Probabilistic Values}

\author{%
  R. Teal Witter \\
  Claremont McKenna College\\
  \texttt{rtealwitter@cmc.edu} \\
  \And
  Yurong Liu \\
  New York University \\
  \texttt{yurong.liu@nyu.edu} \\
  \And
  Christopher Musco \\
  New York University \\
  \texttt{cmusco@nyu.edu} \\
}

\begin{document}

\maketitle

\begin{abstract}
With origins in game theory, probabilistic values like Shapley values, Banzhaf values, and semi-values have emerged as a central tool in explainable AI. They are used for feature attribution, data attribution, data valuation, and more. Since exact computation can require exponential time, research has focused on efficient approximation methods using two techniques: Monte Carlo sampling and linear regression formulations. In this work, we present a new way of combining both of these techniques. Our approach is more flexible than prior algorithms, allowing for linear regression to be replaced with any function family whose probabilistic values can be computed efficiently. This allows us to harness the accuracy of tree-based models like XGBoost; with cross-fitting, the method also produces unbiased estimates. From experiments across eight datasets, we find that our methods give state-of-the-art performance for estimating probabilistic values. For Shapley values, the error of our methods can be $11\times$ lower than Permutation SHAP (the most popular Monte Carlo method), $6.5\times$ lower than Kernel SHAP (the most popular linear regression method), and $3.4\times$ lower than Leverage SHAP (the prior state-of-the-art Shapley value estimator). For more general probabilistic values, we can obtain error $390\times$ lower than the best estimator from prior work.
\end{abstract}

\section{Introduction}

As AI becomes more prevalent across health care, education, finance, and the legal system, underlying algorithmic mechanisms are growing increasingly complex. Sophisticated computational models frequently make decisions that are opaque and challenging to comprehend. This is unacceptable in contexts where decisions can have profound consequences for individuals: the ability to clearly understand and explain how an algorithmic system reaches its conclusions is paramount. 

One tool that has arisen to address the challenge of understanding model behavior are \emph{probabilistic values}, which include Shapley values, Banzhaf values, and semi-values as special cases \cite{strumbelj2010efficient,lundberg2017unified,lundberg2020local,wang2023databanzhaf}.
Originating from game theory \cite{shapley1951notes}, probabilistic values quantify the contribution of a player by measuring how its addition to a set of other players changes the value of the game. 
Formally, consider a \emph{value function} $v: 2^{[n]} \to \mathbb{R}$ defined on sets $S \subseteq [n]$, where $[n]$ denotes $\{1,\ldots, n\}$. The probabilistic value for player $i \in [n]$ is
\begin{align}
\label{eq:prob_value_intro}
\phi_i(v) = \sum_{S \subseteq [n] \setminus \{i\}}
p_{|S|} [ v(S \cup \{i\})-v(S)]
\end{align}
where $\mathbf{p} = [p_0, \ldots, p_{n-1}]\in [0,1]^n$ is a set of probabilistic weights that satisfy $\sum_{\ell=0}^{n-1} \binom{n-1}{\ell} p_\ell=1$. Throughout, we set $p_{-1}=p_n=0$ for boundary terms and omit zero-coefficient terms from sums involving division by sampling probabilities.
We can interpret the $i$th probabilistic value as the \textit{average} marginal contribution of player $i$ to random set $S$, where the distribution over set sizes is specified by $\mathbf{p}$. Different choices of $\mathbf{p}$ yield different variants of probabilistic values \cite{kwon2022beta,kwon2022weightedshap,li2024robust}. For example, to obtain the ubiquitous Shapley values, set $p_{\ell} = \frac1{n} \binom{n-1}{\ell}^{-1}$, and to obtain Banzhaf values, set $p_{\ell} = 1/2^{n-1}$ for all $\ell$. 

Our paper addresses the problem of computing $\phi_1, \ldots, \phi_n$ in full generality for any $\mathbf{p}$. The topic of which weights are best for a given application has received significant attention. Some prior work focuses on axiomatic approaches for choosing $\mathbf{p}$. For example, all probabilistic values satisfy three desirable properties: \textit{null player}, \textit{symmetry}, and \textit{linearity} (see \cite{weber1988probabilistic} for a detailed discussion). Shapley values satisfy an additional \emph{efficiency} property \cite{shapley1951notes} and Banzhaf values satisfy a \emph{2-efficiency} property that might be desirable when there are non-linear interactions between players \cite{penrose1946elementary,banzhaf1964weighted}. Generalizations of these values include Beta Shapley \cite{kwon2022beta} values and weighted Banzhaf values \cite{li2024robust}. See Appendix \ref{appendix:more_probs} for more on these generalizations.

Regardless of how $\mathbf{p}$ is chosen, the meaning of the probabilistic values depends on how the value function, $v$, is defined. For example, a common task in explainable AI is to attribute a model prediction (for a given input) to features \cite{lundberg2017unified}. Here, $v(S)$, is the prediction made when using just the subset of features corresponding to $S$.\footnote{Since most models in machine learning require a full set of input features, features not in $S$ are replaced with either a mean value or random value from the training dataset as a baseline \cite{janzing2020feature,lundberg2018consistent}.}  Probabilistic values are also used in data attribution tasks, where $v(S)$ corresponds to the model loss when training with a given subset of data \cite{ghorbani2019data,wang2025data}. In these applications and others, evaluating $v$ is expensive, as it requires re-running  or possibly even re-training a model. As in prior work on efficient probabilistic value estimation, we thus focus on algorithms that estimate $\phi_1, \ldots, \phi_n$ using as few evaluations of $v$ as possible. We view these evaluations as black-box, designing algorithms that are agnostic to the particular value function $v$, and can thus be applied in a wide range of downstream applications. 

\subsection{Efficiently Computing Probabilistic Values}
For general value functions and probabilistic weights, exact computation can require exponential time, as the summation in Equation \eqref{eq:prob_value_intro} involves $O(2^{n})$ terms. When $v$ is a highly structured function, like a linear function or decision tree, more efficient algorithms exist \cite{lundberg2017unified,lundberg2018consistent,lundberg2020local,karczmarz2022improved}. However, given the complexity of modern machine learning models, most prior work focuses on approximation algorithms.

The standard method is to approximate the summation in Equation \eqref{eq:prob_value_intro} via a Monte Carlo estimate obtained from a weighted sample of sets that do not contain $i$  \cite{kwon2022beta,kwon2022weightedshap,li2024one}. Concretely, assume for simplicity that we sample a collection of subsets, $\mathcal{S}_i$, by drawing samples with replacement from a distribution with density $\mathcal{D}: 2^{[n]\setminus\{i\}} \to [0,1]$.\footnote{In order to efficiently sample, $\mathcal{D}$ typically assigns the same density to subsets of the same size.} We then compute the unbiased estimate:
\begin{align}
\label{eq:basic_mc_method}
\tilde{\phi}_i^{\text{MC}} = \frac1{|\mathcal{S}_i|}
\sum_{S \in \mathcal{S}_i}  [v(S \cup \{i\}) - v(S) ] \frac{p_{|S|}}{\mathcal{D}(S)}
\end{align}
We have that $\E[\tilde{\phi}_i^{\text{MC}}] = \phi_i$, and the estimator's variance depends on the choice of sampling distribution $\mathcal{D}$, as well as $[v(S \cup \{i\}) - v(S)]^2$ for all $S \subseteq [n]$. In addition to high-variance in practice\footnote{For independent samples with finite variance, the standard deviation scales with $1/\sqrt{|\mathcal{S}_i|}$, i.e., only as the inverse square root of the number of samples.}, a downside of Monte Carlo estimators is that it is difficult to ``reuse'' samples between indices $1, \ldots, n$, as each term in Equation \eqref{eq:basic_mc_method} requires evaluating both $v(S \cup \{i\})$ and $v(S)$ for a particular $i$. Several methods address this issue via ``sample reuse'' \cite{castro2009polynomial}. One technique especially relevant to our work is the \emph{maximum sample reuse} (MSR) method, which was originally applied to Banzhaf values \cite{wang2023databanzhaf}, but generalizes naturally to all probabilistic values \cite{kolpaczki2024approximating,li2024faster,li2024one}. The MSR method draws a single collection of subsets, $\mathcal{S}$, according to $\mathcal{D}:2^{[n]}\to [0,1]$, and computes the estimate:
\begin{align}
\label{eq:basic_msr_method}
\tilde{\phi}_i^{\text{MSR}} = \frac1{|\mathcal{S}|}
\sum_{S \in \mathcal{S}} v(S) \frac{p_{|S|-1} \mathbbm{1}[i \in S] - p_{|S|}\mathbbm{1}[i \notin S] }{\mathcal{D}(S)}.
\end{align}
It can be checked that we still have $\E[\tilde{\phi}_i^{\text{MSR}}] = \phi_i$ for all $i$.
Moreover, every evaluation of the value function, $v(S)$, contributes to the estimate for \emph{all} $i \in [n]$, so we achieve maximum sample reuse. However, the variance of MSR methods scales as a weighted sum of $[v(S)]^2$, which is generally much larger than the difference between nearby values $[v(S \cup \{i\}) - v(S)]^2$.

\textbf{Beyond Monte Carlo.} 
Given the high variance of Monte Carlo methods, an alternate approach based on \emph{regression} has become popular for the special case of Shapley values. In particular, Shapley values are the unique solution to a particular overdetermined linear regression problem \cite{charnes1988extremal}: 
\begin{align}\label{eq:linear_regression_connection}
   \boldsymbol{\phi} = [\phi_1, \ldots, \phi_n] =
   \argmin_{\mathbf{x}: \langle \mathbf{x}, \mathbf{1} \rangle = v([n]) - v(\emptyset)} \| \mathbf{Ax - b} \|_W, 
\end{align}
\vspace{-.75em}

where $\mathbf{A} \in \mathbb{R}^{2^n \times n}$  is a specific structured matrix whose rows correspond to sets $S \subseteq [n]$, $\mathbf{b} \in \mathbb{R}^{2^n}$ is vector whose entries equal $v(S) - v(\emptyset)$, and $\|\cdot\|_W$ is a weighted $\ell_2$ norm. 

The ubiquitous Kernel SHAP algorithm \cite{lundberg2017unified,covert2020improving} takes advantage of the regression formulation  by \emph{approximately solving} Equation \eqref{eq:linear_regression_connection} using a subsample of constraints (and corresponding entries in $\mathbf{b}$), each of which requires evaluating $v(S)$ for a single subset $S$. This approach was recently improved by incorporating leverage score sampling \cite{Sarlos:2006,SpielmanSrivastava:2011}, resulting in the state-of-the-art Leverage SHAP method \cite{musco2024provably}. In addition to inherent sample reuse, the empirical effectiveness of Kernel SHAP and Leverage SHAP seems related to the fact that the accuracy of both methods depends on how well $v$ is approximated by a linear function. Indeed, if $v$ is exactly linear, both methods return exact Shapley values once the sampled regression design has full rank on the efficiency-constrained subspace \cite{musco2024provably}. However, even when $v$ is not linear, there are theoretical guarantees on the performance of Kernel SHAP and Leverage SHAP \cite{musco2024provably,chen2025unified}.

The Kernel SHAP approach has been extended to Banzhaf values \cite{liu2025kernelbanzhaffastrobust}, thanks to a similarly elegant regression formulation \cite{hammer1992approximations}. However, extensions to more general probabilistic values have been less effective, failing to outperform Monte Carlo methods \cite{lin2022measuring, li2024faster, li2024one}. A key challenge is that, due to the lack of an efficiency property, generalized linear regression formulations for probabilistic values typically require estimating the \emph{sum} of these values, which introduces another source of error \cite{ruiz1998family}. 
Moreover, even for Shapley and Banzhaf values, a drawback of regression-based methods is that they fail to provide an unbiased estimate for each $\phi_i$. Attempts to fix this issue have generally led to estimates with much higher variance \cite{covert2020improving}.

\begin{figure}
    \centering
    \includegraphics[width=\linewidth]{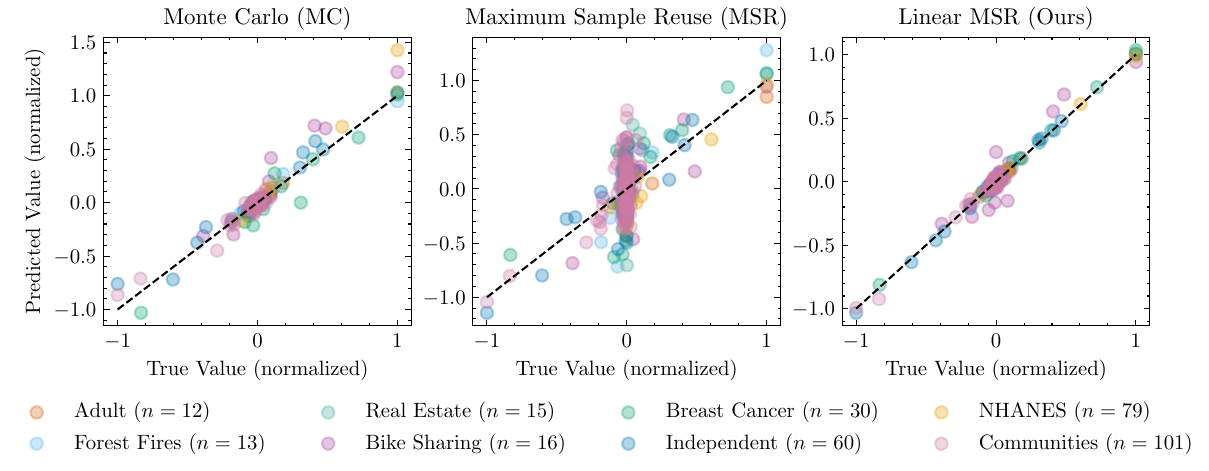}
    \caption{Predicted versus true (normalized) Shapley values for three unbiased estimators (our Regression MSR estimators use cross-fitting; see Section \ref{sec:regMSR}) given a fixed number of black-box evaluations of the value function, $v$. Each point represents one feature's estimated vs true Shapley value on one dataset. The Monte Carlo estimator uses each sample to estimate only one Shapley value, but has variance that depends on the difference in values between neighboring sets, i.e., $[v(S \cup \{i\}) - v(S)]^2$.
    The Maximum Sample Reuse (MSR) estimator reuses samples, but has larger variance that depends on the magnitude of the values, i.e., $[v(S)]^2$.
    Our Regression MSR estimators reuse samples \textit{and} have smaller variance that depends on how well a learned function $f$ fits the value function $v$, i.e., $[v(S) - f(S)]^2$.
    Even taking $f$ to be linear gives excellent performance (we call this method Linear MSR). Taking $f$ to be a decision-tree model (Tree MSR) can produce even better estimates for large sample sizes, as shown in Figure \ref{fig:shapley_sample_size}.}
    \label{fig:case_study}
    \vspace{-.5em}
\end{figure}

\subsection{Our Contributions}
We introduce a method called \emph{Regression MSR} for leveraging regression to approximate probabilistic values. In contrast to previous work on regression methods, Regression MSR easily extends to all probabilistic values, and its cross-fit version produces unbiased estimates.
Moreover, the method can take advantage of non-linear regression methods like XGBoost \cite{chen2016xgboost} and other decision-tree models. 

Instead of starting with a custom linear regression formulation for a given type of probabilistic value, Regression MSR uses regression as a \emph{variance reduction} method for Monte Carlo approximation, and specifically, for the Maximum Sample Reuse method introduced earlier. Concretely, learning from a small number of random subsets, we start by approximating the value function $v$ using a simpler function, $f$. Using the fact that the probabilistic values are linear --- i.e., $\phi_i(v) = \phi_i(f) + \phi_i(v - f)$ for any $f$ --- we propose to return the estimator: 
\begin{align}
\label{eq:msr_intro}
\tilde{\phi}_i%
= \phi_i(f) + \frac1{|\mathcal{S}|}
\sum_{S \in \mathcal{S}} [v(S) - f(S)] \frac{p_{|S|-1} \mathbbm{1}[i \in S] - p_{|S|}\mathbbm{1}[i \notin S] }{\mathcal{D}(S)}.
\end{align}
Since the MSR estimator (second term) is consistent --- i.e., returns the true probabilistic values when run on all subsets --- the Regression MSR estimator is too.
Further, it can be checked that $\E\left[\tilde{\phi}_i%
\right] = \phi_i$ for any fixed $f$, so the method is unbiased whenever $f$ is learned on samples not in $\mathcal{S}$; Section \ref{sec:regMSR} arranges this by cross-fitting. Moreover, like the biased Kernel and Leverage SHAP methods, the variance of $\tilde{\phi}_i%
$ depends on $[v(S) - f(S)]^2$ (see Section \ref{sec:regMSR} for details). So, our method is more accurate than the standard MSR method when we can obtain a good approximation to $v$. 

\begin{figure}[tb!]
    \centering
    \includegraphics[width=\linewidth]{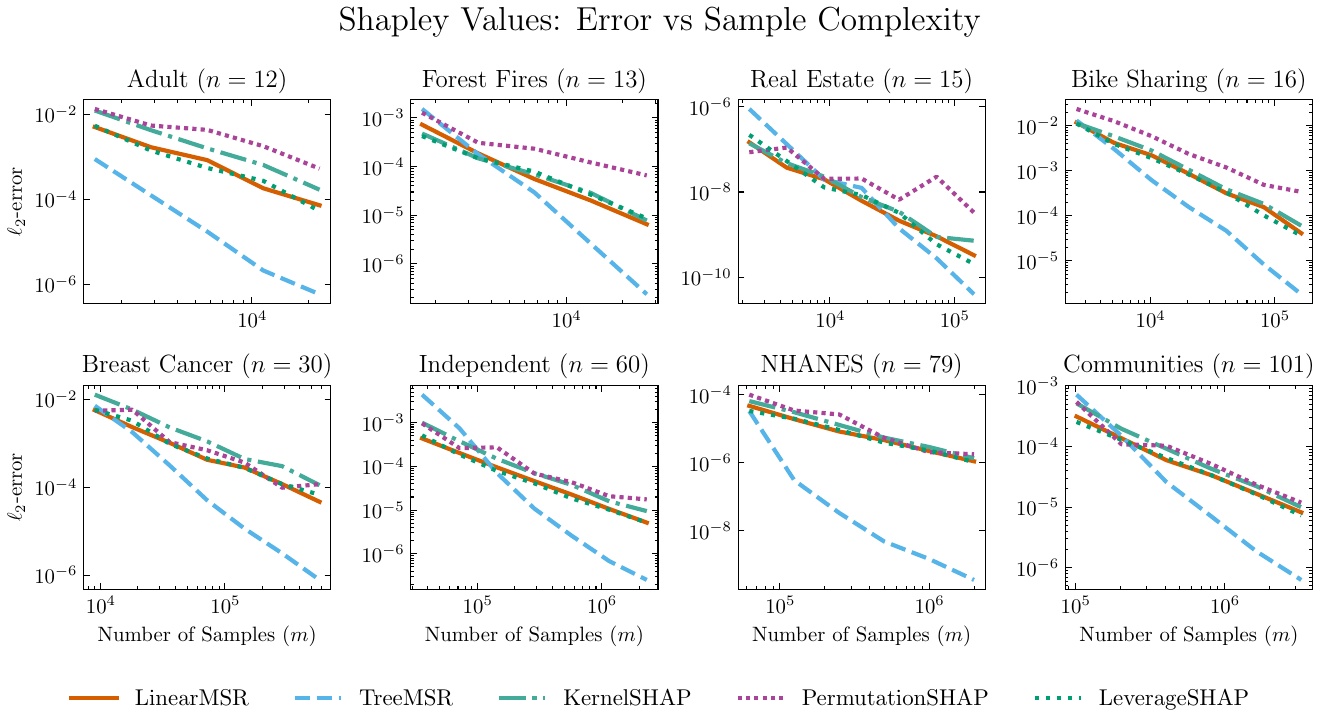}
    \caption{Average relative squared $\ell_2$ error between estimated and true Shapley values as a function of sample size $m$ (number of evaluations of $v$) for various datasets. The lines report the mean error over 10 runs, and $m=10n, 20n, 40n, 80n, 160n, 320n, 640n$. Linear MSR consistently performs comparably to the prior state-of-the-art Leverage SHAP. Meanwhile, the performance of Tree MSR depends on how well the tree-based model approximates the value function; with more samples, it can even outperform Leverage SHAP by several orders of magnitude.}
    \label{fig:shapley_sample_size}
\end{figure}

The benefit of our approach is clear in Figure \ref{fig:case_study}, where we take $f$ to be a linear approximation to $v$, fit with cross-fitting, and use Equation \eqref{eq:msr_intro} to estimate Shapley values. However, there is also the potential to go beyond linear approximations. Observe that, to evaluate $\tilde{\phi}_i%
$, the function $f$ does not need to be a linear. Indeed, we can use \emph{any approximation for which the $\phi_i(f)$ term in Equation \eqref{eq:msr_intro} can be computed efficiently}. That is, any function family that admits efficient probabilistic value computation. Importantly, this includes a wide variety of functions based on decision trees. Concretely, in Appendix \ref{appendix:tree_prob}, we show how to  efficiently compute probabilistic values for any linear mixture of decision trees.\footnote{Efficient methods for computing Shapley and Banzhaf values for decision trees were previously known \cite{lundberg2018consistent,lundberg2020local,karczmarz2022improved}. However, they were based on a particular summation property that does not hold for general probabilistic values; if a value function only has contributions from $n' < n$ players, the Shapley/Banzhaf value on the induced game of those $n'$ players is the same as the Shapley/Banzhaf value on the original value function with all $n$ players. Our approach in Appendix \ref{appendix:tree_prob} is based on an alternative way of viewing tree-based models that avoids the need for this property.}

We leverage this observation to learn tree-based approximations to $v$ using powerful models like XGBoost (for the purposes of our experiments, we call this variant of our algorithm Tree MSR).
For the well-studied Shapley values, we find that Tree MSR achieves state-of-the-art performance, especially when there are enough samples for the tree-based model to learn an accurate fit, see e.g., Figure \ref{fig:shapley_sample_size}.
In particular, Tree MSR can yield estimates with average error that is $3.4\times$ lower than the prior state-of-the-art Leverage SHAP estimator (see Table \ref{tab:shapley}).
For general probabilistic values, Tree MSR gives up to $390\times$ lower average error than the best estimator from prior work (see Figure \ref{fig:prob_complexity_big_n} and Table \ref{tab:probabilistic_big_n}).

Concurrent to our work, \cite{butler2025proxyspex} introduce Proxy SPEX for estimating probabilistic values.
Like Tree MSR, they fit gradient boosted trees to the value function $v$.
However, instead of computing the probabilistic values of the trees, they extract the most influential Fourier terms and compute the probabilistic values of the Fourier representation.
In terms of performance, Proxy SPEX outperforms Kernel SHAP for the low sample regime with budget $m \leq 5n$, but Kernel SHAP is more accurate for moderate and larger sample regimes \cite{butler2025proxyspex}.
Tree MSR underperforms Kernel SHAP (and hence Proxy SPEX) in the low sample regime, but generally outperforms Leverage SHAP in larger sample regimes (see e.g., Figure \ref{fig:shapley_sample_size}).
Proxy SPEX is neither consistent nor unbiased, whereas Regression MSR is consistent and, with cross-fitting, unbiased.

\section{Regression MSR}\label{sec:regMSR}

In this section, we present our Regression MSR method, which combines the benefits of Monte Carlo and regression-based estimators.
In particular, Regression MSR reuses every sample for each estimate (like Maximum Sample Reuse and regression-based methods), achieves lower variance when a learned approximation is accurate (like regression-based methods), and, when cross-fit, produces unbiased estimates (like Monte Carlo methods).
Unlike prior linear regression-based methods, Regression MSR successfully extends to any probabilistic value, and can harness the accuracy of richer function classes like regression trees.

The pseudocode of Regression MSR appears in Algorithm \ref{alg:ours}; its parameter $k$ controls how the samples are shared between training and correction.
In its simplest form ($k=1$), which we call the \emph{all-samples} version, the algorithm draws $m$ subsets from $\mathcal{D}$, learns an approximation $f$ to the value function $v$ on all of them, computes the probabilistic values $\phi_i(f)$ for all $i$ (we use linear or tree-based models so that this is efficient), and corrects them with the MSR estimate of $\phi_i(v - f)$ on the same subsets, as in Equation \eqref{eq:msr_intro}.
The correction makes the estimator consistent: run on all $2^n$ subsets, it returns $\phi_i(v)$ exactly, whatever $f$ is.
It does not make the estimator unbiased, because $f$ is fit to the very subsets it is corrected on.
For the linear function class and Shapley values the correction vanishes identically: the residuals of the weighted least-squares fit are orthogonal to the regressors centered along the constraint direction, and the Shapley correction weights are proportional to these centered regressors, so the all-samples version is exactly Leverage SHAP run on the same samples (Appendix \ref{appendix:all_samples}).
For the tree function class the correction does not vanish, but it is built from in-sample residuals that the tree has already shrunk, so the estimate is pulled toward $\phi_i(f)$.

Cross-fitting ($k \geq 2$) separates the samples used to train each function from the samples used to correct it; this makes the estimator unbiased, and allows us to give strong theoretical guarantees in Theorem \ref{thm:error_bound}.
The algorithm draws $k$ collections of $m/k$ samples each, $\mathcal{S}^{(1)}, \ldots, \mathcal{S}^{(k)}$: independent draws from $\mathcal{D}$ when sampling with replacement, or the $k$ parts of one sample of $m$ distinct subsets when sampling without replacement.
(The two cases differ in one detail of the final step. When sampling without replacement, a subset in $\mathcal{S}^{(\ell)}$ is by construction one that $f^{(\ell)}$ was not trained on, so extrapolating from the held-out residuals to all subsets is biased for a function class that fits its training points closely, such as regression trees: the held-out residuals are systematically larger. Instead, the implementation counts the residuals on the training subsets exactly and extrapolates from the held-out subsets only to the subsets that were never drawn, using the inclusion probabilities conditional on the training subsets; Appendix \ref{appendix:proof} gives the formula and shows that the estimate stays unbiased for every function class.)
For each $\ell$, the algorithm learns an approximation $f^{(\ell)}$ to $v$ on all samples that are \textit{not} in $\mathcal{S}^{(\ell)}$, computes the probabilistic values $\phi_i(f^{(\ell)})$ for all $i$, and corrects them with the MSR estimate on the samples in $\mathcal{S}^{(\ell)}$; the final estimate is the average over $\ell$.
Because $f^{(\ell)}$ does not depend on $\mathcal{S}^{(\ell)}$, each correction is an unbiased estimate of $\phi_i(v - f^{(\ell)})$, so the estimator keeps the variance reduction of Equation \eqref{eq:msr_intro} and drops its bias.

\begin{algorithm}[t]
    \caption{Regression Maximum Sample Reuse}
    \label{alg:ours}
   \begin{algorithmic}[1]
    \STATE {\bfseries Input:} number of players $n$,
    number of samples $m$, value function $v: 2^{[n]} \to \mathbb{R}$, 
    probabilistic weights $\mathbf{p} \in [0,1]^n$, probability density function for sampling
    $\mathcal{D}: 2^{[n]} \to [0,1]$, number of splits $k \geq 1$
    \STATE {\bfseries Output:} Estimated probabilistic values $\tilde{\phi}_1, \ldots, \tilde{\phi}_n$
    \STATE Sample $\mathcal{S}^{(1)}, \ldots, \mathcal{S}^{(k)}$, each consisting of $m/k$ subsets drawn from $\mathcal{D}$ (independently, with replacement; or as the $k$ parts of one sample drawn without replacement, in which case, for $k \geq 2$, the final step uses the estimate in Appendix \ref{appendix:proof}). Let $\mathcal{T}^{(\ell)} \gets \bigcup_{\ell' \neq \ell} \mathcal{S}^{(\ell')}$ be the training samples for fold $\ell$; when $k = 1$, $\mathcal{T}^{(1)} \gets \mathcal{S}^{(1)}$.
     \FOR{$\ell \in \{1,\ldots, k\}$}
         \STATE For $i \in [n]$, initialize $\tilde{\phi}^{(\ell)}_i \gets 0$.
         \STATE Learn $f^{(\ell)}: 2^{[n]} \to \mathbb{R}$ to minimize loss
         \begin{align*}
             \sum_{S \in \mathcal{T}^{(\ell)}} [v(S) - f(S)]^2.
        \end{align*}        
        \STATE For all $i \in [n]$, compute probabilistic values $\phi_i(f^{(\ell)})$. \hfill $\triangleright$ Efficient for linear/tree-based models.
        \STATE For all $i \in [n]$, compute 
        \begin{align*}
            \tilde{\phi}_i^{(\ell)}
            \gets \phi_i(f^{(\ell)}) + \frac{1}{|\mathcal{S}^{(\ell)}|} \sum_{S \in \mathcal{S}^{(\ell)}} [v(S) - f^{(\ell)}(S)] \frac{p_{|S|-1} \mathbbm{1}[i \in S] - p_{|S|} \mathbbm{1}[i \notin S]}{\mathcal{D}(S)}.
        \end{align*}        
    \ENDFOR 
    \STATE For all $i \in [n]$, compute final estimate $\tilde{\phi}_i \gets \frac1{k} \sum_{\ell} \tilde{\phi}_i^{(\ell)}$.
    \STATE \textbf{return} $\tilde{\phi}_1, \ldots, \tilde{\phi}_n$
   \end{algorithmic}
\end{algorithm}

Theorem \ref{thm:error_bound} gives theoretical guarantees on the performance of cross-fit Regression MSR.
For fixed $k$, error constraint $\epsilon > 0$, and failure probability $\delta > 0$, Regression MSR uses a linear number of samples to produce estimates with squared $\ell_2$ error that depends on a natural weighted squared error between the value function and our worst learned function.
We present the guarantee for any sampling distribution $\mathcal{D}$ over subsets, and, below, discuss our suggested choice of this distribution.

\begin{restatable}[Regression-Adjustment Guarantee]{theorem}{errorbound}\label{thm:error_bound}
    Let $k \geq 2$, and assume $\mathcal{D}(S)>0$ whenever $p_{|S|-1}^2 |S| + p_{|S|}^2(n-|S|)>0$. The estimates produced by Algorithm \ref{alg:ours} are unbiased estimates of the probabilistic values.
    Further, let $\epsilon, \delta > 0$, and $f_{\max}$ be the learned function $f^{(\ell)}$, $\ell \in [k]$, with the largest weighted error $\sum_{S \subseteq [n]} [v(S) - f^{(\ell)}(S)]^2 \frac{p_{|S|-1}^2 |S| + p_{|S|}^2 (n-|S|)}{\mathcal{D}(S)}$ (see Appendix \ref{appendix:proof}).
    When run with independent sampling with replacement and $m \geq k^2 n/(\epsilon\delta)$ samples (rounded up to a multiple of $k$), Algorithm \ref{alg:ours} produces estimates that satisfy, with probability $1-\delta$,
    \begin{align}
        \| \tilde{\bm{\phi}} - \bm{\phi}\|_2^2 \leq 
        \epsilon \sum_{S \subseteq [n]}
        [v(S) - f_{\max}(S)]^2 
        \frac{p_{|S|}^2 (1-\frac{|S|}{n}) + p_{|S|-1}^2 \frac{|S|}{n}}
        {\mathcal{D}(S)}.
        \label{eq:bound}
    \end{align}
\end{restatable}

With cross-fitting, Algorithm \ref{alg:ours} can be used to make the estimates of a regression-based estimator unbiased while retaining a variance bound in terms of its residual error.
To this end, we purposefully do not specify the sampling distribution $\mathcal{D}$ or the function class $f$.
We next discuss two choices for how to select the model, $f$, and collect samples.

\textbf{Linear MSR}
The simplest choice for $f$ is a linear model.
As discussed in the introduction, there is extensive prior work on special linear regression formulations for Shapley values \cite{charnes1988extremal,lundberg2017unified,covert2020improving,musco2024provably}.
Using a linear function for variance reduction in MSR offers a natural alternative to these methods, and adds negligible computational overhead, yet tends to show superior performance on most datasets (see, e.g., Table \ref{tab:shapley}). When applying the method to Shapley values specifically, we use the existing state-of-the art Leverage SHAP method (which samples via leverage scores) to fit the learned function. We similarly use the linear regression-based Kernel Banzhaf algorithm \cite{liu2025kernelbanzhaffastrobust} when estimating Banzhaf values.
For general probabilistic values, Linear MSR fits a linear model with the sampling distribution described below, and uses its predictions to adjust the final estimates.

\begin{table}[tb!]
    \centering
    \caption{Summary statistics of the average relative squared $\ell_2$ error between estimated and true Shapley values for all datasets listed in Appendix \ref{appendix:datasets}. All estimators are run with $m=40n$ samples. Tree MSR achieves average error that is $11\times$ lower than Permutation SHAP, $6.5\times$ lower than Kernel SHAP, and $3.4\times$ lower than the prior state-of-the-art Leverage SHAP. We emphasize that Tree MSR gives even better performance for larger sample sizes, as shown in Figure \ref{fig:shapley_sample_size}. We follow Olympic medal convention: \colorbox{gold!60}{gold}, \colorbox{silver!60}{silver} and \colorbox{bronze!60}{bronze} signify first, second and third best performance, respectively.}
    \vspace{.5em}
    \resizebox{\linewidth}{!}{ 
\begin{tabular} {lcccccccc|c}
\hline
 & \textbf{Adult} & \textbf{Forest Fires} & \textbf{Real Estate} & \textbf{Bike Sharing} & \textbf{Breast Cancer} & \textbf{Independent} & \textbf{NHANES} & \textbf{Communities} & \textbf{Mean} \\ 
\hline 
\textbf{LinearMSR} &  &  &  &  &  &  &  &  \\ 
\hspace{7pt}Mean & \cellcolor{bronze!60}$\num{8.29e-04}$ & \cellcolor{silver!60}$\num{5.54e-05}$ & \cellcolor{bronze!60}$\num{1.99e-08}$ & \cellcolor{bronze!60}$\num{2.20e-03}$ & \cellcolor{bronze!60}$\num{1.03e-03}$ & \cellcolor{bronze!60}$\num{9.54e-05}$ & \cellcolor{silver!60}$\num{8.37e-06}$ & \cellcolor{silver!60}$\num{5.92e-05}$ & \cellcolor{bronze!60}$\num{5.35e-04}$ \\ 
\hspace{7pt}1st Quartile & \cellcolor{silver!60}$\num{2.50e-05}$ & $\num{2.06e-05}$ & \cellcolor{bronze!60}$\num{6.66e-11}$ & \cellcolor{silver!60}$\num{4.56e-04}$ & \cellcolor{gold!60}$\num{8.84e-05}$ & \cellcolor{silver!60}$\num{3.96e-05}$ & $\num{9.54e-08}$ & \cellcolor{silver!60}$\num{4.45e-06}$ & \cellcolor{silver!60}$\num{7.93e-05}$ \\ 
\hspace{7pt}2nd Quartile & \cellcolor{silver!60}$\num{5.17e-04}$ & \cellcolor{silver!60}$\num{3.15e-05}$ & \cellcolor{silver!60}$\num{1.27e-09}$ & \cellcolor{silver!60}$\num{1.35e-03}$ & $\num{1.13e-03}$ & \cellcolor{bronze!60}$\num{7.24e-05}$ & $\num{3.51e-06}$ & \cellcolor{bronze!60}$\num{5.43e-05}$ & \cellcolor{silver!60}$\num{3.95e-04}$ \\ 
\hspace{7pt}3rd Quartile & \cellcolor{bronze!60}$\num{1.47e-03}$ & \cellcolor{silver!60}$\num{8.15e-05}$ & \cellcolor{gold!60}$\num{2.83e-09}$ & \cellcolor{bronze!60}$\num{4.23e-03}$ & $\num{1.76e-03}$ & \cellcolor{silver!60}$\num{9.74e-05}$ & \cellcolor{silver!60}$\num{1.50e-05}$ & \cellcolor{silver!60}$\num{8.36e-05}$ & \cellcolor{bronze!60}$\num{9.67e-04}$ \\ 
\hline 
\textbf{TreeMSR} &  &  &  &  &  &  &  &  \\ 
\hspace{7pt}Mean & \cellcolor{gold!60}$\num{1.74e-05}$ & \cellcolor{gold!60}$\num{2.94e-05}$ & \cellcolor{silver!60}$\num{1.96e-08}$ & \cellcolor{gold!60}$\num{6.20e-04}$ & \cellcolor{gold!60}$\num{3.14e-04}$ & \cellcolor{gold!60}$\num{7.06e-05}$ & \cellcolor{gold!60}$\num{3.33e-08}$ & \cellcolor{gold!60}$\num{2.58e-05}$ & \cellcolor{gold!60}$\num{1.35e-04}$ \\ 
\hspace{7pt}1st Quartile & \cellcolor{gold!60}$\num{2.87e-06}$ & \cellcolor{gold!60}$\num{5.24e-06}$ & $\num{1.15e-09}$ & \cellcolor{gold!60}$\num{1.46e-04}$ & $\num{1.65e-04}$ & \cellcolor{gold!60}$\num{1.52e-05}$ & \cellcolor{gold!60}$\num{3.36e-11}$ & $\num{1.44e-05}$ & \cellcolor{gold!60}$\num{4.36e-05}$ \\ 
\hspace{7pt}2nd Quartile & \cellcolor{gold!60}$\num{1.64e-05}$ & \cellcolor{gold!60}$\num{1.30e-05}$ & \cellcolor{bronze!60}$\num{1.77e-09}$ & \cellcolor{gold!60}$\num{5.67e-04}$ & \cellcolor{gold!60}$\num{2.76e-04}$ & \cellcolor{gold!60}$\num{6.13e-05}$ & \cellcolor{gold!60}$\num{6.19e-10}$ & \cellcolor{gold!60}$\num{2.91e-05}$ & \cellcolor{gold!60}$\num{1.20e-04}$ \\ 
\hspace{7pt}3rd Quartile & \cellcolor{gold!60}$\num{2.50e-05}$ & \cellcolor{gold!60}$\num{3.95e-05}$ & \cellcolor{bronze!60}$\num{5.67e-09}$ & \cellcolor{gold!60}$\num{8.70e-04}$ & \cellcolor{gold!60}$\num{4.31e-04}$ & \cellcolor{bronze!60}$\num{9.89e-05}$ & \cellcolor{gold!60}$\num{2.69e-08}$ & \cellcolor{gold!60}$\num{3.38e-05}$ & \cellcolor{gold!60}$\num{1.87e-04}$ \\ 
\hline 
\textbf{KernelSHAP} &  &  &  &  &  &  &  &  \\ 
\hspace{7pt}Mean & $\num{1.57e-03}$ & \cellcolor{bronze!60}$\num{6.95e-05}$ & $\num{2.30e-08}$ & $\num{2.81e-03}$ & $\num{2.30e-03}$ & $\num{1.50e-04}$ & $\num{1.31e-05}$ & $\num{9.15e-05}$ & $\num{8.76e-04}$ \\ 
\hspace{7pt}1st Quartile & $\num{5.90e-04}$ & \cellcolor{bronze!60}$\num{1.86e-05}$ & \cellcolor{silver!60}$\num{3.86e-11}$ & $\num{6.96e-04}$ & $\num{3.79e-04}$ & $\num{8.82e-05}$ & $\num{1.36e-07}$ & \cellcolor{bronze!60}$\num{5.58e-06}$ & $\num{2.22e-04}$ \\ 
\hspace{7pt}2nd Quartile & $\num{1.45e-03}$ & \cellcolor{bronze!60}$\num{4.62e-05}$ & $\num{1.78e-09}$ & $\num{2.27e-03}$ & $\num{2.58e-03}$ & $\num{9.53e-05}$ & $\num{4.56e-06}$ & $\num{8.79e-05}$ & $\num{8.16e-04}$ \\ 
\hspace{7pt}3rd Quartile & $\num{2.18e-03}$ & \cellcolor{bronze!60}$\num{1.11e-04}$ & \cellcolor{silver!60}$\num{4.58e-09}$ & $\num{4.92e-03}$ & $\num{3.95e-03}$ & $\num{2.01e-04}$ & $\num{2.46e-05}$ & $\num{1.52e-04}$ & $\num{1.44e-03}$ \\ 
\hline 
\textbf{PermutationSHAP} &  &  &  &  &  &  &  &  \\ 
\hspace{7pt}Mean & $\num{4.26e-03}$ & $\num{2.33e-04}$ & $\num{2.04e-08}$ & $\num{5.87e-03}$ & $\num{1.07e-03}$ & $\num{2.74e-04}$ & $\num{2.67e-05}$ & $\num{1.03e-04}$ & $\num{1.48e-03}$ \\ 
\hspace{7pt}1st Quartile & $\num{2.96e-04}$ & $\num{7.35e-05}$ & $\num{7.90e-11}$ & $\num{1.56e-03}$ & \cellcolor{bronze!60}$\num{1.00e-04}$ & $\num{8.88e-05}$ & \cellcolor{bronze!60}$\num{8.38e-08}$ & $\num{7.24e-06}$ & $\num{2.66e-04}$ \\ 
\hspace{7pt}2nd Quartile & $\num{1.81e-03}$ & $\num{9.19e-05}$ & $\num{2.78e-09}$ & $\num{4.31e-03}$ & \cellcolor{silver!60}$\num{5.56e-04}$ & $\num{2.13e-04}$ & \cellcolor{silver!60}$\num{1.54e-06}$ & \cellcolor{silver!60}$\num{4.22e-05}$ & $\num{8.78e-04}$ \\ 
\hspace{7pt}3rd Quartile & $\num{2.59e-03}$ & $\num{4.17e-04}$ & $\num{1.93e-08}$ & $\num{6.47e-03}$ & \cellcolor{silver!60}$\num{1.42e-03}$ & $\num{2.80e-04}$ & $\num{2.90e-05}$ & \cellcolor{bronze!60}$\num{8.67e-05}$ & $\num{1.41e-03}$ \\ 
\hline 
\textbf{LeverageSHAP} &  &  &  &  &  &  &  &  \\ 
\hspace{7pt}Mean & \cellcolor{silver!60}$\num{5.43e-04}$ & $\num{7.87e-05}$ & \cellcolor{gold!60}$\num{1.29e-08}$ & \cellcolor{silver!60}$\num{1.91e-03}$ & \cellcolor{silver!60}$\num{9.73e-04}$ & \cellcolor{silver!60}$\num{7.89e-05}$ & \cellcolor{bronze!60}$\num{9.29e-06}$ & \cellcolor{bronze!60}$\num{6.45e-05}$ & \cellcolor{silver!60}$\num{4.57e-04}$ \\ 
\hspace{7pt}1st Quartile & \cellcolor{bronze!60}$\num{3.64e-05}$ & \cellcolor{silver!60}$\num{1.66e-05}$ & \cellcolor{gold!60}$\num{2.00e-11}$ & \cellcolor{bronze!60}$\num{5.67e-04}$ & \cellcolor{silver!60}$\num{9.80e-05}$ & \cellcolor{bronze!60}$\num{3.97e-05}$ & \cellcolor{silver!60}$\num{8.23e-08}$ & \cellcolor{gold!60}$\num{3.61e-06}$ & \cellcolor{bronze!60}$\num{9.51e-05}$ \\ 
\hspace{7pt}2nd Quartile & \cellcolor{bronze!60}$\num{5.54e-04}$ & $\num{4.76e-05}$ & \cellcolor{gold!60}$\num{1.26e-09}$ & \cellcolor{bronze!60}$\num{1.64e-03}$ & \cellcolor{bronze!60}$\num{1.12e-03}$ & \cellcolor{silver!60}$\num{6.51e-05}$ & \cellcolor{bronze!60}$\num{3.40e-06}$ & $\num{6.44e-05}$ & \cellcolor{bronze!60}$\num{4.37e-04}$ \\ 
\hspace{7pt}3rd Quartile & \cellcolor{silver!60}$\num{9.72e-04}$ & $\num{1.14e-04}$ & $\num{5.77e-09}$ & \cellcolor{silver!60}$\num{2.47e-03}$ & \cellcolor{bronze!60}$\num{1.63e-03}$ & \cellcolor{gold!60}$\num{9.21e-05}$ & \cellcolor{bronze!60}$\num{1.71e-05}$ & $\num{1.13e-04}$ & \cellcolor{silver!60}$\num{6.76e-04}$ \\ 
\hline
\end{tabular}}
    \label{tab:shapley}
\end{table}

\textbf{Tree MSR}
Beyond linear models, tree-based regression models like XGBoost can learn more accurate approximations.
As discussed in the introduction, it is known how to efficiently compute the Shapley values \cite{lundberg2020local} and Banzhaf values \cite{karczmarz2022improved} of trees; however, it was previously unclear how to generalize these approaches to probabilistic values:
the Shapley/Banzhaf methods use the property that the probabilistic value of a value function on $n$ players is the probabilistic value of the extended value function on $n' > n$ players, as long as the additional $n'-n$ players always contribute nothing.
This property unfortunately does not hold for all probabilistic values; e.g., consider probabilistic weights $\mathbf{p}$ that are independently sampled (and normalized) for each number of players $n$ and $n'$.
Instead, efficiently computing the probabilistic values of trees requires a subtly different approach, which we describe in Appendix \ref{appendix:tree_prob}.
With this approach in hand, we can fit $v$ with a tree-based model like XGBoost and efficiently compute its probabilistic values for the final estimate.

\textbf{Sampling Distribution}
When Algorithm \ref{alg:ours} is not run on top of another regression-based estimator, we choose the sampling distribution so that the function $f$ is directly trained to minimize the error bound in Theorem \ref{thm:error_bound}.
That is, we sample each set with probability proportional to
\begin{align*}
    \sqrt{p_{|S|}^2 (1 - \frac{|S|}{n}) + p_{|S|-1}^2 \frac{|S|}{n}}.
\end{align*}
Then the error bound in Theorem \ref{thm:error_bound} is proportional to
\begin{align*}
    \sum_{S \subseteq [n]}
    [v(S) - f(S)]^2 
    \sqrt{p_{|S|}^2 (1 - \frac{|S|}{n}) + p_{|S|-1}^2 \frac{|S|}{n}}
\end{align*}
which, by design, is the \textit{expected} loss used to train $f$.

\paragraph{Choosing $k$}
With $k=1$, Regression MSR trains one function on all $m$ samples, which is the cheapest version of the algorithm.
With $k \geq 2$, it trains $k$ functions, each on a $(k-1)/k$ fraction of the samples, and corrects each on the held-out $1/k$, which makes the estimates unbiased; increasing $k$ gives each function a larger training set and may improve accuracy, but raises computational cost.
In our experiments, Linear MSR and Tree MSR denote the cross-fit estimator with $k=10$, so that each function is trained on $90\%$ of the samples.
Appendix \ref{appendix:all_samples} compares it with the all-samples version on the same experiments.
The all-samples version runs $1.1$--$5.8\times$ faster with the linear function class and $2.6$--$9.9\times$ faster with the tree function class.
The cross-fit estimator is more accurate in 53 of the 112 dataset--weighting pairs for the linear class, so the two versions are nearly interchangeable there (with Shapley weights, cross-fitting is the only difference between Linear MSR and Leverage SHAP), and in 90 of the 112 pairs for the tree class, where the all-samples version's error is typically $1.8\times$ larger.
We recommend the all-samples version when runtime dominates and cross-fitting otherwise.

\begin{figure}[tb!]
    \centering
    \includegraphics[width=\linewidth]{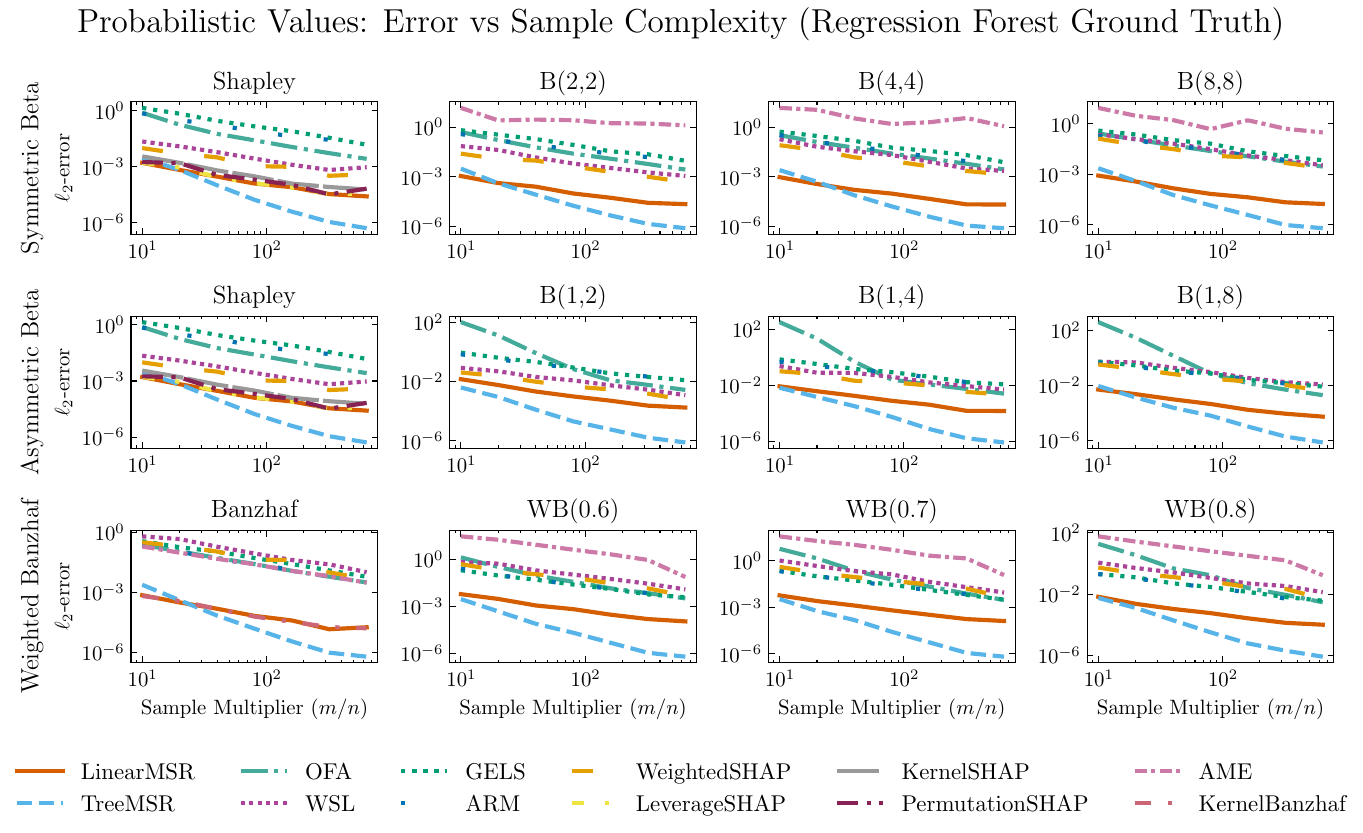}
    \caption{Average error between the estimated and true probabilistic values as a function of sample size. Each subplot shows results for a different probabilistic value with the error averaged over all large datasets ($n \geq 30$), for which we used the tree-based method described in Appendix \ref{appendix:tree_prob}. The lines report the mean error over 10 runs. Tree MSR gives the best performance, often by several orders of magnitude when $m$ is large.}
    \label{fig:prob_complexity_big_n}
\end{figure}

\section{Experiments}\label{sec:experiments}

In this section, we describe our experiments on eight datasets.
Overall, we find that Linear MSR and Tree MSR give state-of-the-art performance for almost all datasets and sample budgets.\footnote{The code is available at \url{https://github.com/rtealwitter/regressionMSR}.}

\textbf{Value Function}
For evaluation, we focus on the explainable AI feature attribution task, but emphasize that our methods can be applied to any application involving probabilistic values, as we only require black-box access to the value function $v$.
Concretely, we train a model on a dataset, and attribute the prediction the model makes on a given \textit{explicand} point $\mathbf{x}^e \in \mathbb{R}^n$ to its $n$ input features.
We consider the \textit{interventional} definition of $v$, where the explanation is relative to a \textit{baseline} point $\mathbf{x}^b \in \mathbb{R}^n$:
For a set $S$, let $\mathbf{x}^S$ be the point where the $i$th feature is $x_i^e$ if $i \in S$ and $x_i^b$ otherwise.
Then, the value function $v(S)$ is the model's prediction on $\mathbf{x}^S$.
There is also a conditional version of feature attribution, where the features not in $S$ are drawn from their conditional distribution given the observed features in $S$ \cite{lundberg2017unified,lundberg2018consistent}.
However, we choose to focus on the interventional version since it is more efficient to compute $v(S)$, and the resulting probabilistic values are more interpretable \cite{janzing2020feature}.

\begin{table}[tb!]
    \centering
    \caption{Summary statistics of the relative squared $\ell_2$ error between estimated and true probabilistic values when $m=40n$. We summarize the error over large datasets $(n \geq 30)$, for which we use the tree-based method described in Appendix \ref{appendix:tree_prob} to compute the true probabilistic values. On average over all probabilistic values, Tree MSR produces estimates with mean error that is $390\times$ lower than the best estimator from prior work.}
    \vspace{.5em}
    \resizebox{\linewidth}{!}{ 
\begin{tabular} {lcccccccccccc|c}
\hline
 & \textbf{B(1,1)} & \textbf{B(2,2)} & \textbf{B(4,4)} & \textbf{B(8,8)} & \textbf{B(1,2)} & \textbf{B(1,4)} & \textbf{B(1,8)} & \textbf{WB(0.5)} & \textbf{WB(0.6)} & \textbf{WB(0.7)} & \textbf{WB(0.8)} & \textbf{WB(0.9)} & \textbf{Mean} \\ 
\hline 
\textbf{LinearMSR} &  &  &  &  &  &  &  &  &  &  &  &  \\ 
\hspace{7pt}Mean & \cellcolor{silver!60}$\num{3.23e-04}$ & \cellcolor{silver!60}$\num{2.44e-04}$ & \cellcolor{silver!60}$\num{1.56e-04}$ & \cellcolor{silver!60}$\num{1.48e-04}$ & \cellcolor{silver!60}$\num{2.07e-03}$ & \cellcolor{silver!60}$\num{1.71e-03}$ & \cellcolor{silver!60}$\num{9.66e-04}$ & \cellcolor{silver!60}$\num{1.46e-04}$ & \cellcolor{silver!60}$\num{1.12e-03}$ & \cellcolor{silver!60}$\num{1.30e-03}$ & \cellcolor{silver!60}$\num{1.10e-03}$ & \cellcolor{silver!60}$\num{8.96e-04}$ & \cellcolor{silver!60}$\num{8.48e-04}$ \\ 
\hspace{7pt}1st Quartile & \cellcolor{silver!60}$\num{7.29e-06}$ & \cellcolor{silver!60}$\num{6.07e-06}$ & \cellcolor{silver!60}$\num{7.85e-06}$ & \cellcolor{silver!60}$\num{7.01e-06}$ & \cellcolor{silver!60}$\num{4.33e-04}$ & \cellcolor{silver!60}$\num{2.84e-04}$ & \cellcolor{silver!60}$\num{1.74e-04}$ & \cellcolor{silver!60}$\num{5.52e-06}$ & \cellcolor{silver!60}$\num{2.33e-04}$ & \cellcolor{silver!60}$\num{1.92e-04}$ & \cellcolor{silver!60}$\num{1.60e-04}$ & \cellcolor{silver!60}$\num{8.69e-05}$ & \cellcolor{silver!60}$\num{1.33e-04}$ \\ 
\hspace{7pt}2nd Quartile & \cellcolor{silver!60}$\num{4.24e-05}$ & \cellcolor{silver!60}$\num{4.80e-05}$ & \cellcolor{silver!60}$\num{4.25e-05}$ & \cellcolor{silver!60}$\num{3.45e-05}$ & \cellcolor{silver!60}$\num{7.63e-04}$ & \cellcolor{silver!60}$\num{5.94e-04}$ & \cellcolor{silver!60}$\num{3.33e-04}$ & \cellcolor{silver!60}$\num{3.79e-05}$ & \cellcolor{silver!60}$\num{6.75e-04}$ & \cellcolor{silver!60}$\num{5.81e-04}$ & \cellcolor{silver!60}$\num{4.30e-04}$ & \cellcolor{silver!60}$\num{2.52e-04}$ & \cellcolor{silver!60}$\num{3.19e-04}$ \\ 
\hspace{7pt}3rd Quartile & \cellcolor{silver!60}$\num{1.35e-04}$ & \cellcolor{silver!60}$\num{1.27e-04}$ & \cellcolor{silver!60}$\num{1.13e-04}$ & \cellcolor{silver!60}$\num{8.53e-05}$ & \cellcolor{silver!60}$\num{1.82e-03}$ & \cellcolor{silver!60}$\num{1.32e-03}$ & \cellcolor{silver!60}$\num{7.11e-04}$ & \cellcolor{silver!60}$\num{8.46e-05}$ & \cellcolor{silver!60}$\num{1.34e-03}$ & \cellcolor{silver!60}$\num{1.30e-03}$ & \cellcolor{silver!60}$\num{1.17e-03}$ & \cellcolor{silver!60}$\num{4.88e-04}$ & \cellcolor{silver!60}$\num{7.23e-04}$ \\ 
\hline 
\textbf{TreeMSR} &  &  &  &  &  &  &  &  &  &  &  &  \\ 
\hspace{7pt}Mean & \cellcolor{gold!60}$\num{8.89e-05}$ & \cellcolor{gold!60}$\num{8.05e-05}$ & \cellcolor{gold!60}$\num{7.35e-05}$ & \cellcolor{gold!60}$\num{5.92e-05}$ & \cellcolor{gold!60}$\num{1.23e-04}$ & \cellcolor{gold!60}$\num{3.15e-04}$ & \cellcolor{gold!60}$\num{2.39e-04}$ & \cellcolor{gold!60}$\num{6.15e-05}$ & \cellcolor{gold!60}$\num{7.53e-05}$ & \cellcolor{gold!60}$\num{1.48e-04}$ & \cellcolor{gold!60}$\num{2.00e-04}$ & \cellcolor{gold!60}$\num{3.77e-04}$ & \cellcolor{gold!60}$\num{1.53e-04}$ \\ 
\hspace{7pt}1st Quartile & \cellcolor{gold!60}$\num{1.26e-06}$ & \cellcolor{gold!60}$\num{7.62e-07}$ & \cellcolor{gold!60}$\num{4.65e-07}$ & \cellcolor{gold!60}$\num{4.31e-07}$ & \cellcolor{gold!60}$\num{7.90e-07}$ & \cellcolor{gold!60}$\num{2.53e-06}$ & \cellcolor{gold!60}$\num{1.49e-05}$ & \cellcolor{gold!60}$\num{5.73e-07}$ & \cellcolor{gold!60}$\num{6.19e-07}$ & \cellcolor{gold!60}$\num{1.90e-06}$ & \cellcolor{gold!60}$\num{2.37e-06}$ & \cellcolor{gold!60}$\num{3.29e-06}$ & \cellcolor{gold!60}$\num{2.49e-06}$ \\ 
\hspace{7pt}2nd Quartile & \cellcolor{gold!60}$\num{2.91e-05}$ & \cellcolor{gold!60}$\num{1.83e-05}$ & \cellcolor{gold!60}$\num{1.73e-05}$ & \cellcolor{gold!60}$\num{1.54e-05}$ & \cellcolor{gold!60}$\num{2.62e-05}$ & \cellcolor{gold!60}$\num{4.69e-05}$ & \cellcolor{gold!60}$\num{8.66e-05}$ & \cellcolor{gold!60}$\num{1.11e-05}$ & \cellcolor{gold!60}$\num{1.43e-05}$ & \cellcolor{gold!60}$\num{1.33e-05}$ & \cellcolor{gold!60}$\num{2.02e-05}$ & \cellcolor{gold!60}$\num{2.52e-05}$ & \cellcolor{gold!60}$\num{2.70e-05}$ \\ 
\hspace{7pt}3rd Quartile & \cellcolor{gold!60}$\num{1.07e-04}$ & \cellcolor{gold!60}$\num{6.69e-05}$ & \cellcolor{gold!60}$\num{7.11e-05}$ & \cellcolor{gold!60}$\num{5.17e-05}$ & \cellcolor{gold!60}$\num{8.01e-05}$ & \cellcolor{gold!60}$\num{2.85e-04}$ & \cellcolor{gold!60}$\num{3.29e-04}$ & \cellcolor{gold!60}$\num{5.16e-05}$ & \cellcolor{gold!60}$\num{4.78e-05}$ & \cellcolor{gold!60}$\num{6.68e-05}$ & \cellcolor{gold!60}$\num{1.57e-04}$ & \cellcolor{gold!60}$\num{1.36e-04}$ & \cellcolor{gold!60}$\num{1.21e-04}$ \\ 
\hline 
\textbf{OFA} &  &  &  &  &  &  &  &  &  &  &  &  \\ 
\hspace{7pt}Mean & $\num{6.02e-02}$ & $\num{5.89e-02}$ & $\num{5.65e-02}$ & $\num{5.25e-02}$ & $\num{8.18e-01}$ & $\num{4.72e-01}$ & $\num{1.41e+00}$ & $\num{5.26e-02}$ & $\num{9.26e-02}$ & $\num{2.36e-01}$ & $\num{4.88e-01}$ & $\num{1.66e+00}$ & $\num{4.55e-01}$ \\ 
\hspace{7pt}1st Quartile & $\num{5.06e-02}$ & $\num{5.15e-02}$ & $\num{4.89e-02}$ & $\num{4.74e-02}$ & $\num{5.29e-02}$ & $\num{5.25e-02}$ & $\num{5.81e-02}$ & $\num{4.56e-02}$ & $\num{5.05e-02}$ & $\num{5.61e-02}$ & $\num{6.34e-02}$ & $\num{5.12e-02}$ & $\num{5.24e-02}$ \\ 
\hspace{7pt}2nd Quartile & $\num{6.28e-02}$ & $\num{5.87e-02}$ & $\num{5.48e-02}$ & $\num{5.24e-02}$ & $\num{6.85e-02}$ & $\num{7.89e-02}$ & $\num{1.08e-01}$ & $\num{5.24e-02}$ & $\num{5.87e-02}$ & $\num{7.29e-02}$ & $\num{8.25e-02}$ & $\num{1.10e-01}$ & $\num{7.17e-02}$ \\ 
\hspace{7pt}3rd Quartile & $\num{7.01e-02}$ & $\num{6.49e-02}$ & $\num{6.54e-02}$ & $\num{5.79e-02}$ & $\num{2.60e-01}$ & $\num{1.38e-01}$ & $\num{2.29e-01}$ & \cellcolor{bronze!60}$\num{5.86e-02}$ & $\num{7.08e-02}$ & $\num{9.40e-02}$ & $\num{1.55e-01}$ & $\num{3.08e-01}$ & $\num{1.31e-01}$ \\ 
\hline 
\textbf{WSL} &  &  &  &  &  &  &  &  &  &  &  &  \\ 
\hspace{7pt}Mean & $\num{6.37e-03}$ & $\num{1.45e-02}$ & $\num{3.43e-02}$ & $\num{6.70e-02}$ & $\num{1.99e-02}$ & $\num{7.44e-02}$ & $\num{1.88e-01}$ & $\num{1.55e-01}$ & $\num{1.96e-01}$ & $\num{2.19e-01}$ & $\num{2.66e-01}$ & $\num{2.55e-01}$ & $\num{1.25e-01}$ \\ 
\hspace{7pt}1st Quartile & $\num{5.00e-04}$ & $\num{5.05e-03}$ & $\num{1.30e-02}$ & $\num{3.46e-02}$ & $\num{6.04e-03}$ & $\num{1.85e-02}$ & $\num{3.73e-02}$ & $\num{6.03e-02}$ & $\num{6.46e-02}$ & $\num{8.22e-02}$ & $\num{6.55e-02}$ & $\num{8.77e-02}$ & $\num{3.96e-02}$ \\ 
\hspace{7pt}2nd Quartile & $\num{1.42e-03}$ & $\num{9.09e-03}$ & $\num{2.26e-02}$ & $\num{5.13e-02}$ & $\num{1.29e-02}$ & $\num{4.83e-02}$ & $\num{8.38e-02}$ & $\num{1.15e-01}$ & $\num{1.44e-01}$ & $\num{1.28e-01}$ & $\num{1.36e-01}$ & $\num{1.97e-01}$ & $\num{7.92e-02}$ \\ 
\hspace{7pt}3rd Quartile & $\num{4.31e-03}$ & $\num{2.05e-02}$ & $\num{4.59e-02}$ & $\num{7.90e-02}$ & $\num{2.76e-02}$ & $\num{1.03e-01}$ & $\num{2.16e-01}$ & $\num{1.79e-01}$ & $\num{2.21e-01}$ & $\num{2.86e-01}$ & $\num{3.06e-01}$ & $\num{3.38e-01}$ & $\num{1.52e-01}$ \\ 
\hline 
\textbf{GELS} &  &  &  &  &  &  &  &  &  &  &  &  \\ 
\hspace{7pt}Mean & $\num{2.54e-01}$ & $\num{1.87e-01}$ & $\num{1.48e-01}$ & $\num{1.08e-01}$ & $\num{2.06e-01}$ & $\num{1.43e-01}$ & $\num{1.18e-01}$ & $\num{1.14e-01}$ & $\num{4.96e-02}$ & $\num{5.45e-02}$ & $\num{5.12e-02}$ & $\num{6.01e-02}$ & $\num{1.24e-01}$ \\ 
\hspace{7pt}1st Quartile & $\num{1.45e-01}$ & $\num{9.06e-02}$ & $\num{6.44e-02}$ & $\num{6.48e-02}$ & $\num{9.52e-02}$ & $\num{8.18e-02}$ & $\num{6.25e-02}$ & $\num{5.96e-02}$ & \cellcolor{bronze!60}$\num{2.77e-02}$ & \cellcolor{bronze!60}$\num{3.09e-02}$ & \cellcolor{bronze!60}$\num{3.10e-02}$ & \cellcolor{bronze!60}$\num{2.97e-02}$ & $\num{6.53e-02}$ \\ 
\hspace{7pt}2nd Quartile & $\num{1.97e-01}$ & $\num{1.45e-01}$ & $\num{1.00e-01}$ & $\num{7.59e-02}$ & $\num{1.18e-01}$ & $\num{1.32e-01}$ & $\num{8.62e-02}$ & $\num{7.17e-02}$ & $\num{3.59e-02}$ & \cellcolor{bronze!60}$\num{3.54e-02}$ & \cellcolor{bronze!60}$\num{4.25e-02}$ & $\num{4.89e-02}$ & $\num{9.07e-02}$ \\ 
\hspace{7pt}3rd Quartile & $\num{3.07e-01}$ & $\num{2.34e-01}$ & $\num{1.86e-01}$ & $\num{1.18e-01}$ & $\num{2.14e-01}$ & $\num{1.72e-01}$ & $\num{1.53e-01}$ & $\num{1.45e-01}$ & $\num{6.67e-02}$ & $\num{7.43e-02}$ & $\num{6.66e-02}$ & $\num{6.95e-02}$ & $\num{1.51e-01}$ \\ 
\hline 
\textbf{ARM} &  &  &  &  &  &  &  &  &  &  &  &  \\ 
\hspace{7pt}Mean & $\num{1.54e-01}$ & $\num{7.98e-02}$ & $\num{6.23e-02}$ & $\num{4.94e-02}$ & $\num{1.62e-01}$ & $\num{1.27e-01}$ & $\num{1.10e-01}$ & \cellcolor{bronze!60}$\num{5.14e-02}$ & \cellcolor{bronze!60}$\num{4.90e-02}$ & \cellcolor{bronze!60}$\num{4.79e-02}$ & \cellcolor{bronze!60}$\num{4.95e-02}$ & \cellcolor{bronze!60}$\num{4.96e-02}$ & $\num{8.27e-02}$ \\ 
\hspace{7pt}1st Quartile & $\num{8.10e-02}$ & $\num{4.31e-02}$ & $\num{3.49e-02}$ & $\num{3.15e-02}$ & $\num{6.97e-02}$ & $\num{8.13e-02}$ & $\num{8.18e-02}$ & \cellcolor{bronze!60}$\num{2.94e-02}$ & $\num{2.94e-02}$ & $\num{3.23e-02}$ & $\num{3.91e-02}$ & $\num{3.83e-02}$ & $\num{4.93e-02}$ \\ 
\hspace{7pt}2nd Quartile & $\num{1.22e-01}$ & $\num{5.95e-02}$ & $\num{4.38e-02}$ & $\num{3.84e-02}$ & $\num{8.86e-02}$ & $\num{1.03e-01}$ & $\num{9.63e-02}$ & \cellcolor{bronze!60}$\num{3.90e-02}$ & \cellcolor{bronze!60}$\num{3.58e-02}$ & $\num{3.90e-02}$ & $\num{4.54e-02}$ & \cellcolor{bronze!60}$\num{4.49e-02}$ & $\num{6.30e-02}$ \\ 
\hspace{7pt}3rd Quartile & $\num{1.82e-01}$ & $\num{9.58e-02}$ & $\num{6.95e-02}$ & $\num{5.73e-02}$ & $\num{1.73e-01}$ & $\num{1.52e-01}$ & $\num{1.19e-01}$ & $\num{6.71e-02}$ & \cellcolor{bronze!60}$\num{4.76e-02}$ & \cellcolor{bronze!60}$\num{4.95e-02}$ & \cellcolor{bronze!60}$\num{5.97e-02}$ & \cellcolor{bronze!60}$\num{5.72e-02}$ & $\num{9.42e-02}$ \\ 
\hline 
\textbf{WeightedSHAP} &  &  &  &  &  &  &  &  &  &  &  &  \\ 
\hspace{7pt}Mean & \cellcolor{bronze!60}$\num{2.68e-03}$ & \cellcolor{bronze!60}$\num{9.33e-03}$ & \cellcolor{bronze!60}$\num{1.47e-02}$ & \cellcolor{bronze!60}$\num{3.18e-02}$ & \cellcolor{bronze!60}$\num{9.58e-03}$ & \cellcolor{bronze!60}$\num{2.05e-02}$ & \cellcolor{bronze!60}$\num{6.55e-02}$ & $\num{7.82e-02}$ & $\num{1.11e-01}$ & $\num{9.11e-02}$ & $\num{1.32e-01}$ & $\num{1.50e-01}$ & \cellcolor{bronze!60}$\num{5.97e-02}$ \\ 
\hspace{7pt}1st Quartile & \cellcolor{bronze!60}$\num{3.70e-04}$ & \cellcolor{bronze!60}$\num{2.59e-03}$ & \cellcolor{bronze!60}$\num{4.10e-03}$ & \cellcolor{bronze!60}$\num{1.30e-02}$ & \cellcolor{bronze!60}$\num{3.60e-03}$ & \cellcolor{bronze!60}$\num{6.49e-03}$ & \cellcolor{bronze!60}$\num{3.19e-02}$ & $\num{3.35e-02}$ & $\num{3.09e-02}$ & $\num{4.30e-02}$ & $\num{5.33e-02}$ & $\num{4.79e-02}$ & \cellcolor{bronze!60}$\num{2.26e-02}$ \\ 
\hspace{7pt}2nd Quartile & \cellcolor{bronze!60}$\num{9.56e-04}$ & \cellcolor{bronze!60}$\num{5.36e-03}$ & \cellcolor{bronze!60}$\num{1.02e-02}$ & \cellcolor{bronze!60}$\num{2.50e-02}$ & \cellcolor{bronze!60}$\num{5.77e-03}$ & \cellcolor{bronze!60}$\num{1.46e-02}$ & \cellcolor{bronze!60}$\num{5.23e-02}$ & $\num{6.31e-02}$ & $\num{4.76e-02}$ & $\num{7.90e-02}$ & $\num{8.94e-02}$ & $\num{1.15e-01}$ & \cellcolor{bronze!60}$\num{4.23e-02}$ \\ 
\hspace{7pt}3rd Quartile & \cellcolor{bronze!60}$\num{1.92e-03}$ & \cellcolor{bronze!60}$\num{9.01e-03}$ & \cellcolor{bronze!60}$\num{2.10e-02}$ & \cellcolor{bronze!60}$\num{4.31e-02}$ & \cellcolor{bronze!60}$\num{1.30e-02}$ & \cellcolor{bronze!60}$\num{3.06e-02}$ & \cellcolor{bronze!60}$\num{9.60e-02}$ & $\num{8.72e-02}$ & $\num{1.04e-01}$ & $\num{1.21e-01}$ & $\num{2.02e-01}$ & $\num{2.11e-01}$ & \cellcolor{bronze!60}$\num{7.84e-02}$ \\ 
\hline
\end{tabular}}
    \label{tab:probabilistic_big_n}
\end{table}

\textbf{Ground Truth Probabilistic Values}
For small datasets with $n < 30$, we use a neural network model and compute the true probabilistic values through enumeration.
For larger datasets with $n \geq 30$ where exact enumeration is infeasible, we use a random forest model and compute the true probabilistic values using the algorithm described in Appendix \ref{appendix:tree_prob}.
(Please see Appendix \ref{appendix:datasets} for a summary of the datasets in our experiments.)
We emphasize that our method for computing the probabilistic values of trees enables the first experiments on medium to large datasets where we can compare the estimates to the \textit{ground truth} probabilistic values.
Such experiments were previously done for Shapley and Banzhaf values \cite{musco2024provably,liu2025kernelbanzhaffastrobust}, but not for other probabilistic values.

\textbf{Baselines}
We compare Linear MSR and Tree MSR (the cross-fit estimator with $k=10$; Appendix \ref{appendix:all_samples} reports the all-samples version) to a wide variety of probabilistic value estimators from prior work.
For the popular task of estimating Shapley values, we focus on the most effective estimators for general value functions i.e., Permutation SHAP \cite{castro2009polynomial}, Kernel SHAP \cite{lundberg2017unified,covert2020improving}, and Leverage SHAP \cite{musco2024provably}.
We use the optimized implementations of Permutation SHAP and Kernel SHAP in the ubiquitous SHAP library for parity \cite{lundberg2017unified}.
For estimating probabilistic values, there has been substantial recent interest in designing estimators \cite{kwon2022beta,kwon2022weightedshap,lin2022measuring,wang2023databanzhaf,kolpaczki2024approximating,li2024faster,li2024one}.
These estimators generally use the standard Monte Carlo approach, apply the Maximum Sample Reuse idea, or extend linear regression-based methods.
We provide a description of each in Appendix \ref{appendix:baselines}.

\textbf{Error and Uncertainty}
We measure the error between the true probabilistic values $\bm{\phi}$ and the estimated probabilistic values $\tilde{\bm{\phi}}$ with the relative squared $\ell_2$ error $\|\bm{\phi}-\tilde{\bm{\phi}}\|_2^2/\|\bm{\phi}\|_2^2$.
All of our tables and figures report summary statistics over at least 10 runs.
In the tables, we report the mean, first quartile, median, and third quartiles of the error. (We do not report $+/-$ standard deviation because these are often negative for the small errors in our experiments.)
In the figures, we report the mean error.

\textbf{Implementation Details}
We use scikit-learn \cite{scikit-learn} and XGBoost \cite{chen2016xgboost} for training our models.
For the implementations of Permutation SHAP and Kernel SHAP, we use the SHAP library \cite{lundberg2017unified}.
Please see Appendix \ref{appendix:datasets} for details on how we accessed each dataset.
The original experiments were run on a machine with an Apple M2 chip and 8GB RAM. The updated main estimator comparisons and all-samples comparisons were rerun on the Hopper cluster, with one CPU and 16GB of memory allocated per experimental task.

We first describe our experiments on the popular task of estimating Shapley values.
Figure \ref{fig:shapley_sample_size} shows estimator performance by sample size, and Table \ref{tab:shapley} highlights the corresponding uncertainty statistics when each estimator is run with $m=40n$ samples.
We find that Linear MSR, which applies cross-fitting to the same Shapley-weighted linear regression that Leverage SHAP uses and is therefore unbiased, performs comparably to the prior state-of-the-art Leverage SHAP: averaged over the datasets its error is $17\%$ higher, and it is lower on three of the eight datasets.
In contrast, the performance Tree MSR depends on how well the tree-based approximation $f$ fits the underlying value function $v$.
When the number of samples is smaller, Tree MSR is comparable to prior Shapley value estimators; however, as the number of samples grows, the tree-based model becomes more accurate and Tree MSR often gives the best performance, sometimes by orders of magnitude.
As can be seen in Table \ref{tab:shapley}, \textbf{Tree MSR can give average error that is $3.4\times$ lower than the prior state-of-the-art Leverage SHAP}, when $m=40n$.
For larger sample sizes, Tree MSR outperforms Leverage SHAP by an even wider margin, as shown in Figure \ref{fig:shapley_sample_size}.

Beyond estimating Shapley values, we run experiments on estimating the more general beta Shapley values and weighted Banzhaf values (see Appendix \ref{appendix:more_probs} for definitions).
Figure \ref{fig:prob_complexity_big_n} shows estimator performance by sample size for large datasets (where we use tree models so that the true probabilistic values can be computed exactly), and Table \ref{tab:probabilistic_big_n} highlights the corresponding uncertainty statistics when each estimator is run with $m=40n$ samples.
We present the analogous results for smaller datasets (where we can exactly compute the probabilistic values of neural networks) in Figure \ref{fig:prob_complexity_small_n} and Table \ref{tab:probabilistic_small_n} in Appendix \ref{appendix:net_experiments}.
We find that Linear MSR generally plateaus as the number of samples increases.
In contrast, Tree MSR gives the best performance across the board, with the gap to the next best estimator widening with the number of samples.
We confirm this finding in Figure \ref{fig:fit_sample_complexity}.
As can be seen in Table \ref{tab:probabilistic_big_n}, \textbf{Tree MSR can give average error that is $390\times$ lower than the best probabilistic value estimator from prior work} (and $243\times$ lower on the small datasets, Table \ref{tab:probabilistic_small_n}).

We provide additional experiments on the effect of noisy access to the value function in Appendix \ref{appendix:experiments_noise}. Figures \ref{fig:prob_complexity_small_n_noise} and \ref{fig:prob_complexity_big_n_noise} suggest that Tree MSR is particularly resilient to noise.

\begin{figure}[tb!]
    \centering
    \includegraphics[width=\linewidth]{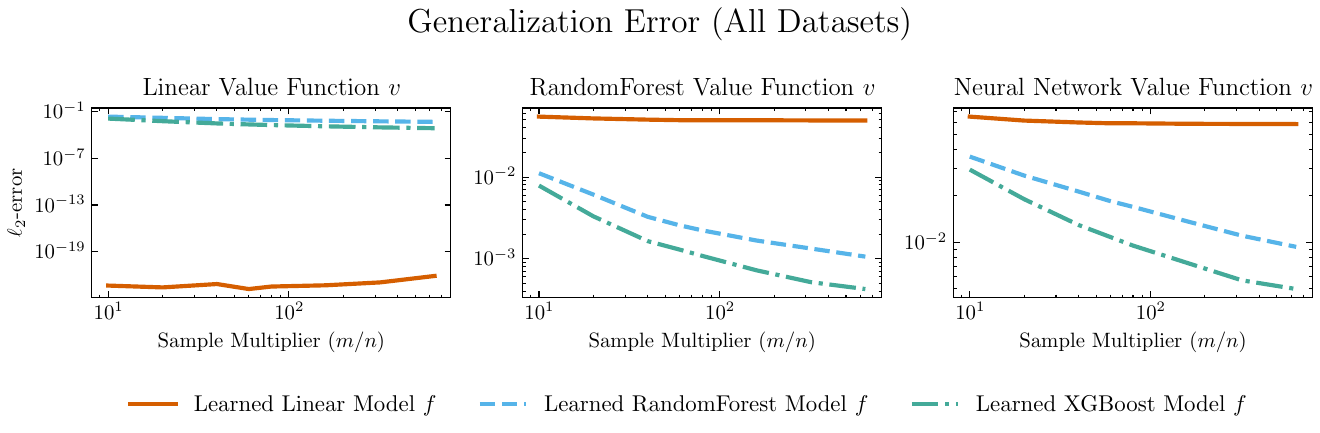}
    \caption{Generalization error between value function $v$ and learned model $f$ by sample size, averaged over all datasets. When the base model is linear, the learned linear model quickly fits it to machine precision. When the base model is a random forest or a neural network, the error of the linear model plateaus while the random forest and XGBoost learned models continue to improve. This phenomenon is reflected in Figures \ref{fig:prob_complexity_big_n} and \ref{fig:prob_complexity_small_n}; the performance of Tree MSR continues to improve with the number of samples while Linear MSR plateaus.}
    \label{fig:fit_sample_complexity}
\end{figure}

\section*{Limitations and Broader Impacts}

The performance of our methods depends on the underlying fit of the learned model.
When the dataset is structured or the number of samples is large, the learned model is accurate and the estimators are, too.
However, for datasets with less structure or in sample-constrained settings, the performance of our estimators can worsen.

The primary application of our work is in explainable AI, where we seek to understand how features and data points contribute to the performance of machine learning models.
We expect the broader impact of our work to be better model explanations, and we do not see substantial negative risks as a result of our research.

\begin{ack}
Witter was supported by NSF Graduate Research Fellowship Grant No. DGE-2234660. Liu was partially supported by NSF Awards IIS-2106888 and the DARPA ASKEM and ARPA-H BDF programs. Musco was partially supported by NSF Award CCF-2045590.
\end{ack}

\bibliographystyle{alpha}
\bibliography{references}

\clearpage

\section*{NeurIPS Paper Checklist}

\begin{enumerate}

\item {\bf Claims}
    \item[] Question: Do the main claims made in the abstract and introduction accurately reflect the paper's contributions and scope?
    \item[] Answer: \answerYes{} %
    \item[] Justification: {Please see Sections \ref{sec:regMSR} and \ref{sec:experiments}, as well as Appendix \ref{appendix:proof}.}
    \item[] Guidelines:
    \begin{itemize}
        \item The answer NA means that the abstract and introduction do not include the claims made in the paper.
        \item The abstract and/or introduction should clearly state the claims made, including the contributions made in the paper and important assumptions and limitations. A No or NA answer to this question will not be perceived well by the reviewers. 
        \item The claims made should match theoretical and experimental results, and reflect how much the results can be expected to generalize to other settings. 
        \item It is fine to include aspirational goals as motivation as long as it is clear that these goals are not attained by the paper. 
    \end{itemize}

\item {\bf Limitations}
    \item[] Question: Does the paper discuss the limitations of the work performed by the authors?
    \item[] Answer: \answerYes{} %
    \item[] Justification: Please see the Limitations and Broader Impacts section immediately before the references.
    \item[] Guidelines:
    \begin{itemize}
        \item The answer NA means that the paper has no limitation while the answer No means that the paper has limitations, but those are not discussed in the paper. 
        \item The authors are encouraged to create a separate "Limitations" section in their paper.
        \item The paper should point out any strong assumptions and how robust the results are to violations of these assumptions (e.g., independence assumptions, noiseless settings, model well-specification, asymptotic approximations only holding locally). The authors should reflect on how these assumptions might be violated in practice and what the implications would be.
        \item The authors should reflect on the scope of the claims made, e.g., if the approach was only tested on a few datasets or with a few runs. In general, empirical results often depend on implicit assumptions, which should be articulated.
        \item The authors should reflect on the factors that influence the performance of the approach. For example, a facial recognition algorithm may perform poorly when image resolution is low or images are taken in low lighting. Or a speech-to-text system might not be used reliably to provide closed captions for online lectures because it fails to handle technical jargon.
        \item The authors should discuss the computational efficiency of the proposed algorithms and how they scale with dataset size.
        \item If applicable, the authors should discuss possible limitations of their approach to address problems of privacy and fairness.
        \item While the authors might fear that complete honesty about limitations might be used by reviewers as grounds for rejection, a worse outcome might be that reviewers discover limitations that aren't acknowledged in the paper. The authors should use their best judgment and recognize that individual actions in favor of transparency play an important role in developing norms that preserve the integrity of the community. Reviewers will be specifically instructed to not penalize honesty concerning limitations.
    \end{itemize}

\item {\bf Theory assumptions and proofs}
    \item[] Question: For each theoretical result, does the paper provide the full set of assumptions and a complete (and correct) proof?
    \item[] Answer: \answerYes{} %
    \item[] Justification: Please see Appendix \ref{appendix:proof}.
    \item[] Guidelines:
    \begin{itemize}
        \item The answer NA means that the paper does not include theoretical results. 
        \item All the theorems, formulas, and proofs in the paper should be numbered and cross-referenced.
        \item All assumptions should be clearly stated or referenced in the statement of any theorems.
        \item The proofs can either appear in the main paper or the supplemental material, but if they appear in the supplemental material, the authors are encouraged to provide a short proof sketch to provide intuition. 
        \item Inversely, any informal proof provided in the core of the paper should be complemented by formal proofs provided in appendix or supplemental material.
        \item Theorems and Lemmas that the proof relies upon should be properly referenced. 
    \end{itemize}

    \item {\bf Experimental result reproducibility}
    \item[] Question: Does the paper fully disclose all the information needed to reproduce the main experimental results of the paper to the extent that it affects the main claims and/or conclusions of the paper (regardless of whether the code and data are provided or not)?
    \item[] Answer: \answerYes{} %
    \item[] Justification: Please see Section \ref{sec:experiments}.
    \item[] Guidelines:
    \begin{itemize}
        \item The answer NA means that the paper does not include experiments.
        \item If the paper includes experiments, a No answer to this question will not be perceived well by the reviewers: Making the paper reproducible is important, regardless of whether the code and data are provided or not.
        \item If the contribution is a dataset and/or model, the authors should describe the steps taken to make their results reproducible or verifiable. 
        \item Depending on the contribution, reproducibility can be accomplished in various ways. For example, if the contribution is a novel architecture, describing the architecture fully might suffice, or if the contribution is a specific model and empirical evaluation, it may be necessary to either make it possible for others to replicate the model with the same dataset, or provide access to the model. In general. releasing code and data is often one good way to accomplish this, but reproducibility can also be provided via detailed instructions for how to replicate the results, access to a hosted model (e.g., in the case of a large language model), releasing of a model checkpoint, or other means that are appropriate to the research performed.
        \item While NeurIPS does not require releasing code, the conference does require all submissions to provide some reasonable avenue for reproducibility, which may depend on the nature of the contribution. For example
        \begin{enumerate}
            \item If the contribution is primarily a new algorithm, the paper should make it clear how to reproduce that algorithm.
            \item If the contribution is primarily a new model architecture, the paper should describe the architecture clearly and fully.
            \item If the contribution is a new model (e.g., a large language model), then there should either be a way to access this model for reproducing the results or a way to reproduce the model (e.g., with an open-source dataset or instructions for how to construct the dataset).
            \item We recognize that reproducibility may be tricky in some cases, in which case authors are welcome to describe the particular way they provide for reproducibility. In the case of closed-source models, it may be that access to the model is limited in some way (e.g., to registered users), but it should be possible for other researchers to have some path to reproducing or verifying the results.
        \end{enumerate}
    \end{itemize}

\item {\bf Open access to data and code}
    \item[] Question: Does the paper provide open access to the data and code, with sufficient instructions to faithfully reproduce the main experimental results, as described in supplemental material?
    \item[] Answer: \answerYes{} %
    \item[] Justification: The code is available at \url{https://github.com/rtealwitter/regressionMSR}.
    \item[] Guidelines:
    \begin{itemize}
        \item The answer NA means that paper does not include experiments requiring code.
        \item Please see the NeurIPS code and data submission guidelines (\url{https://nips.cc/public/guides/CodeSubmissionPolicy}) for more details.
        \item While we encourage the release of code and data, we understand that this might not be possible, so “No” is an acceptable answer. Papers cannot be rejected simply for not including code, unless this is central to the contribution (e.g., for a new open-source benchmark).
        \item The instructions should contain the exact command and environment needed to run to reproduce the results. See the NeurIPS code and data submission guidelines (\url{https://nips.cc/public/guides/CodeSubmissionPolicy}) for more details.
        \item The authors should provide instructions on data access and preparation, including how to access the raw data, preprocessed data, intermediate data, and generated data, etc.
        \item The authors should provide scripts to reproduce all experimental results for the new proposed method and baselines. If only a subset of experiments are reproducible, they should state which ones are omitted from the script and why.
        \item At submission time, to preserve anonymity, the authors should release anonymized versions (if applicable).
        \item Providing as much information as possible in supplemental material (appended to the paper) is recommended, but including URLs to data and code is permitted.
    \end{itemize}

\item {\bf Experimental setting/details}
    \item[] Question: Does the paper specify all the training and test details (e.g., data splits, hyperparameters, how they were chosen, type of optimizer, etc.) necessary to understand the results?
    \item[] Answer: \answerYes{} %
    \item[] Justification: Please see Section \ref{sec:experiments}.
    \item[] Guidelines:
    \begin{itemize}
        \item The answer NA means that the paper does not include experiments.
        \item The experimental setting should be presented in the core of the paper to a level of detail that is necessary to appreciate the results and make sense of them.
        \item The full details can be provided either with the code, in appendix, or as supplemental material.
    \end{itemize}

\item {\bf Experiment statistical significance}
    \item[] Question: Does the paper report error bars suitably and correctly defined or other appropriate information about the statistical significance of the experiments?
    \item[] Answer: \answerYes{} %
    \item[] Justification: Please see Section \ref{sec:experiments}.
    \item[] Guidelines:
    \begin{itemize}
        \item The answer NA means that the paper does not include experiments.
        \item The authors should answer "Yes" if the results are accompanied by error bars, confidence intervals, or statistical significance tests, at least for the experiments that support the main claims of the paper.
        \item The factors of variability that the error bars are capturing should be clearly stated (for example, train/test split, initialization, random drawing of some parameter, or overall run with given experimental conditions).
        \item The method for calculating the error bars should be explained (closed form formula, call to a library function, bootstrap, etc.)
        \item The assumptions made should be given (e.g., Normally distributed errors).
        \item It should be clear whether the error bar is the standard deviation or the standard error of the mean.
        \item It is OK to report 1-sigma error bars, but one should state it. The authors should preferably report a 2-sigma error bar than state that they have a 96\% CI, if the hypothesis of Normality of errors is not verified.
        \item For asymmetric distributions, the authors should be careful not to show in tables or figures symmetric error bars that would yield results that are out of range (e.g. negative error rates).
        \item If error bars are reported in tables or plots, The authors should explain in the text how they were calculated and reference the corresponding figures or tables in the text.
    \end{itemize}

\item {\bf Experiments compute resources}
    \item[] Question: For each experiment, does the paper provide sufficient information on the computer resources (type of compute workers, memory, time of execution) needed to reproduce the experiments?
    \item[] Answer: \answerYes{} %
    \item[] Justification: Please see Section \ref{sec:experiments}.
    \item[] Guidelines:
    \begin{itemize}
        \item The answer NA means that the paper does not include experiments.
        \item The paper should indicate the type of compute workers CPU or GPU, internal cluster, or cloud provider, including relevant memory and storage.
        \item The paper should provide the amount of compute required for each of the individual experimental runs as well as estimate the total compute. 
        \item The paper should disclose whether the full research project required more compute than the experiments reported in the paper (e.g., preliminary or failed experiments that didn't make it into the paper). 
    \end{itemize}
    
\item {\bf Code of ethics}
    \item[] Question: Does the research conducted in the paper conform, in every respect, with the NeurIPS Code of Ethics \url{https://neurips.cc/public/EthicsGuidelines}?
    \item[] Answer: \answerYes{} %
    \item[] Justification: The research conforms to the NeurIPS Code of Ethics.
    \item[] Guidelines:
    \begin{itemize}
        \item The answer NA means that the authors have not reviewed the NeurIPS Code of Ethics.
        \item If the authors answer No, they should explain the special circumstances that require a deviation from the Code of Ethics.
        \item The authors should make sure to preserve anonymity (e.g., if there is a special consideration due to laws or regulations in their jurisdiction).
    \end{itemize}

\item {\bf Broader impacts}
    \item[] Question: Does the paper discuss both potential positive societal impacts and negative societal impacts of the work performed?
    \item[] Answer: \answerYes{} %
    \item[] Justification: {Please see the Limitations and Broader Impacts section before the references.}
    \item[] Guidelines:
    \begin{itemize}
        \item The answer NA means that there is no societal impact of the work performed.
        \item If the authors answer NA or No, they should explain why their work has no societal impact or why the paper does not address societal impact.
        \item Examples of negative societal impacts include potential malicious or unintended uses (e.g., disinformation, generating fake profiles, surveillance), fairness considerations (e.g., deployment of technologies that could make decisions that unfairly impact specific groups), privacy considerations, and security considerations.
        \item The conference expects that many papers will be foundational research and not tied to particular applications, let alone deployments. However, if there is a direct path to any negative applications, the authors should point it out. For example, it is legitimate to point out that an improvement in the quality of generative models could be used to generate deepfakes for disinformation. On the other hand, it is not needed to point out that a generic algorithm for optimizing neural networks could enable people to train models that generate Deepfakes faster.
        \item The authors should consider possible harms that could arise when the technology is being used as intended and functioning correctly, harms that could arise when the technology is being used as intended but gives incorrect results, and harms following from (intentional or unintentional) misuse of the technology.
        \item If there are negative societal impacts, the authors could also discuss possible mitigation strategies (e.g., gated release of models, providing defenses in addition to attacks, mechanisms for monitoring misuse, mechanisms to monitor how a system learns from feedback over time, improving the efficiency and accessibility of ML).
    \end{itemize}
    
\item {\bf Safeguards}
    \item[] Question: Does the paper describe safeguards that have been put in place for responsible release of data or models that have a high risk for misuse (e.g., pretrained language models, image generators, or scraped datasets)?
    \item[] Answer: \answerNA{} %
    \item[] Justification: We do not have datasets or models that have a high risk of misuse.
    \item[] Guidelines:
    \begin{itemize}
        \item The answer NA means that the paper poses no such risks.
        \item Released models that have a high risk for misuse or dual-use should be released with necessary safeguards to allow for controlled use of the model, for example by requiring that users adhere to usage guidelines or restrictions to access the model or implementing safety filters. 
        \item Datasets that have been scraped from the Internet could pose safety risks. The authors should describe how they avoided releasing unsafe images.
        \item We recognize that providing effective safeguards is challenging, and many papers do not require this, but we encourage authors to take this into account and make a best faith effort.
    \end{itemize}

\item {\bf Licenses for existing assets}
    \item[] Question: Are the creators or original owners of assets (e.g., code, data, models), used in the paper, properly credited and are the license and terms of use explicitly mentioned and properly respected?
    \item[] Answer: \answerYes{} %
    \item[] Justification: Please see Appendix \ref{appendix:datasets}.
    \item[] Guidelines:
    \begin{itemize}
        \item The answer NA means that the paper does not use existing assets.
        \item The authors should cite the original paper that produced the code package or dataset.
        \item The authors should state which version of the asset is used and, if possible, include a URL.
        \item The name of the license (e.g., CC-BY 4.0) should be included for each asset.
        \item For scraped data from a particular source (e.g., website), the copyright and terms of service of that source should be provided.
        \item If assets are released, the license, copyright information, and terms of use in the package should be provided. For popular datasets, \url{paperswithcode.com/datasets} has curated licenses for some datasets. Their licensing guide can help determine the license of a dataset.
        \item For existing datasets that are re-packaged, both the original license and the license of the derived asset (if it has changed) should be provided.
        \item If this information is not available online, the authors are encouraged to reach out to the asset's creators.
    \end{itemize}

\item {\bf New assets}
    \item[] Question: Are new assets introduced in the paper well documented and is the documentation provided alongside the assets?
    \item[] Answer: \answerYes{} %
    \item[] Justification: Please see the supplementary material.
    \item[] Guidelines:
    \begin{itemize}
        \item The answer NA means that the paper does not release new assets.
        \item Researchers should communicate the details of the dataset/code/model as part of their submissions via structured templates. This includes details about training, license, limitations, etc. 
        \item The paper should discuss whether and how consent was obtained from people whose asset is used.
        \item At submission time, remember to anonymize your assets (if applicable). You can either create an anonymized URL or include an anonymized zip file.
    \end{itemize}

\item {\bf Crowdsourcing and research with human subjects}
    \item[] Question: For crowdsourcing experiments and research with human subjects, does the paper include the full text of instructions given to participants and screenshots, if applicable, as well as details about compensation (if any)? 
    \item[] Answer: \answerNA{} %
    \item[] Justification: We did not conduct research with human subjects.
    \item[] Guidelines:
    \begin{itemize}
        \item The answer NA means that the paper does not involve crowdsourcing nor research with human subjects.
        \item Including this information in the supplemental material is fine, but if the main contribution of the paper involves human subjects, then as much detail as possible should be included in the main paper. 
        \item According to the NeurIPS Code of Ethics, workers involved in data collection, curation, or other labor should be paid at least the minimum wage in the country of the data collector. 
    \end{itemize}

\item {\bf Institutional review board (IRB) approvals or equivalent for research with human subjects}
    \item[] Question: Does the paper describe potential risks incurred by study participants, whether such risks were disclosed to the subjects, and whether Institutional Review Board (IRB) approvals (or an equivalent approval/review based on the requirements of your country or institution) were obtained?
    \item[] Answer: \answerNA{} %
    \item[] Justification: We did not conduct research with human subjects.
    \item[] Guidelines:
    \begin{itemize}
        \item The answer NA means that the paper does not involve crowdsourcing nor research with human subjects.
        \item Depending on the country in which research is conducted, IRB approval (or equivalent) may be required for any human subjects research. If you obtained IRB approval, you should clearly state this in the paper. 
        \item We recognize that the procedures for this may vary significantly between institutions and locations, and we expect authors to adhere to the NeurIPS Code of Ethics and the guidelines for their institution. 
        \item For initial submissions, do not include any information that would break anonymity (if applicable), such as the institution conducting the review.
    \end{itemize}

\item {\bf Declaration of LLM usage}
    \item[] Question: Does the paper describe the usage of LLMs if it is an important, original, or non-standard component of the core methods in this research? Note that if the LLM is used only for writing, editing, or formatting purposes and does not impact the core methodology, scientific rigorousness, or originality of the research, declaration is not required.
    \item[] Answer: \answerNA{} %
    \item[] Justification: The core method development in this work does not involve LLMs.
    \item[] Guidelines:
    \begin{itemize}
        \item The answer NA means that the core method development in this research does not involve LLMs as any important, original, or non-standard components.
        \item Please refer to our LLM policy (\url{https://neurips.cc/Conferences/2025/LLM}) for what should or should not be described.
    \end{itemize}

\end{enumerate}

\clearpage

\appendix

\section{Proof of Error Bound}\label{appendix:proof}

\errorbound*

\begin{proof}[Proof of Theorem \ref{thm:error_bound}]
Throughout, $k \geq 2$, so that each $f^{(\ell)}$ is trained on subsets outside $\mathcal{S}^{(\ell)}$; with $k=1$ the correction is evaluated on the training subsets themselves, and the argument below does not apply (see Appendix \ref{appendix:all_samples}).
We analyze the estimator when each collection $\mathcal{S}^{(\ell)}$ consists of $m/k$ subsets drawn independently with replacement from $\mathcal{D}$, and the $k$ collections are independent of one another.
Our implementation (for $n < 50$) instead draws one sample of distinct subsets and cross-fits on it, as follows.
Every sampling distribution we use is uniform over the subsets of a given size, so it is specified by per-size weights; for every size $s$ the implementation draws $m_s$ \emph{distinct} subsets of size $s$ uniformly at random, where $m_s := \min\{\binom{n}{s}, \max\{1, \mathrm{round}(x_s)\}\}$, $x_s$ is proportional to the weight of size $s$, and the constant of proportionality is chosen so that the $m_s$ sum to $m$ up to rounding.
In words, every size is drawn at least once, so every subset has a positive inclusion probability $m_{|S|} / \binom{n}{|S|}$, and a size whose target reaches $\binom{n}{s}$ is enumerated, with the budget it does not need going to the other sizes; for our sampling distributions the smallest and largest sizes, which carry the most weight per subset, are enumerated once $m$ is moderately large.
We take $\mathcal{D}$ in Theorem \ref{thm:error_bound} to be the distribution the implementation actually samples from, $\mathcal{D}(S) := m_{|S|} / (\bar{m} \binom{n}{|S|})$ for the subsets of the sizes that are not enumerated, where $\bar{m} \leq m$ is the number of drawn subsets of those sizes; the enumerated sizes contribute no variance at all. In this implementation-specific comparison, all weighted-error sums are restricted to the non-enumerated sizes; if every size is enumerated, the correction is exact and no variance bound is needed.
Within every size, the drawn subsets are then split into $k$ parts as evenly as possible, in random order; a fold whose part of some size is empty holds out one drawn subset of that size chosen uniformly at random instead, unless that size is enumerated, in which case there is nothing left to extrapolate to.
Fold $\ell$ trains $f^{(\ell)}$ on the training subsets $\mathcal{T}^{(\ell)}$, the sample minus its held-out part $\mathcal{S}^{(\ell)}$, and its estimate is
\begin{align}
    \tilde{\phi}_i^{(\ell)}
    &\gets \phi_i(f^{(\ell)})
    + \sum_{S \in \mathcal{T}^{(\ell)}} [v(S) - f^{(\ell)}(S)] \, a_i(S)
    + \sum_{S \in \mathcal{S}^{(\ell)}} [v(S) - f^{(\ell)}(S)] \frac{a_i(S)}{\pi^{(\ell)}_S},
    \nonumber\\
    a_i(S) &:= p_{|S|-1} \mathbbm{1}[i \in S] - p_{|S|} \mathbbm{1}[i \notin S],
    \label{eq:wor_estimate}
\end{align}
where, for $|S| = s$, $\pi^{(\ell)}_S := t_s / \left(\binom{n}{s} - m_s + t_s\right)$ and $t_s$ is the number of subsets of size $s$ in $\mathcal{S}^{(\ell)}$.
In words, the residuals on the training subsets are known and are counted exactly, and the held-out subsets stand in only for the subsets that nobody drew.
(For symmetric weights $\mathbf{p}$ the implementation also draws every subset together with its complement and keeps the two in the same part, which leaves all of the probabilities above unchanged; we do not analyze this standard variance-reduction device.)

Equation \eqref{eq:wor_estimate} is unbiased conditionally on $\mathcal{T}^{(\ell)}$, which is all that the rest of the proof uses.
Fix a size $s$ and condition on the $m_s - t_s$ subsets of size $s$ in $\mathcal{T}^{(\ell)}$.
Every pair (training part, held-out part) of sizes $(m_s - t_s, t_s)$ is equally likely, so given the training part the held-out subsets of size $s$ are a uniform sample without replacement of size $t_s$ from the $\binom{n}{s} - m_s + t_s$ subsets of size $s$ not in $\mathcal{T}^{(\ell)}$, each of which is therefore held out with probability exactly $\pi^{(\ell)}_S$.
Hence, for every function $g$ of subsets,
\begin{align*}
    \E \left[ \sum_{S \in \mathcal{S}^{(\ell)}} \frac{g(S)}{\pi^{(\ell)}_S} \,\middle|\, \mathcal{T}^{(\ell)} \right]
    &= \sum_{S \notin \mathcal{T}^{(\ell)}} g(S),\\
    \E \left[ \sum_{S \in \mathcal{T}^{(\ell)}} g(S) + \sum_{S \in \mathcal{S}^{(\ell)}} \frac{g(S)}{\pi^{(\ell)}_S} \,\middle|\, \mathcal{T}^{(\ell)} \right]
    &= \sum_{S \subseteq [n]} g(S).
\end{align*}
With $g(S) = [v(S) - f^{(\ell)}(S)] \, a_i(S)$, which is a fixed function given $\mathcal{T}^{(\ell)}$ because $f^{(\ell)}$ is trained on $\mathcal{T}^{(\ell)}$ alone, this is the second line of the unbiasedness calculation below, for every function class.
(For the linear function class with Shapley weights the implementation samples only the sizes $1, \ldots, n-1$ and fits $f^{(\ell)}$ without an intercept under the efficiency constraint $\sum_i \phi_i(f^{(\ell)}) = v([n]) - v(\emptyset)$, as in Appendix \ref{appendix:all_samples}.
The two terms of the sum over all $S$ that this omits, $[v(\emptyset) - f^{(\ell)}(\emptyset)]\, a_i(\emptyset) + [v([n]) - f^{(\ell)}([n])]\, a_i([n]) = \big( [v([n]) - f^{(\ell)}([n])] - [v(\emptyset) - f^{(\ell)}(\emptyset)] \big) / n$, vanish identically under the constraint, so the omitted sizes contribute neither bias nor variance; for the same reason the implementation's final step for this class, which shifts every coordinate by $(v([n]) - v(\emptyset) - \sum_i \phi_i(f^{(\ell)})) / n$, is the identity in exact arithmetic.)
(Extrapolating from the held-out subsets to \emph{all} subsets with the unconditional inclusion probabilities $m_{|S|} / \binom{n}{|S|}$, as a naive split of the sample would, is biased: a held-out subset is by construction one that $f^{(\ell)}$ was not trained on, and for a function class that fits its training points closely, such as regression trees, the held-out residuals are systematically larger than the residuals over all subsets.)
For the variance comparison in this paragraph, consider the unpaired stratified design. Conditionally on $\mathcal{T}^{(\ell)}$ only the held-out sum is random; consider one size $s$ at a time (different sizes are independent, and an enumerated size contributes no variance at all). Pairing preserves the unbiasedness argument but introduces covariance between complement subsets; we do not claim the following quantitative variance comparison for the paired implementation.
The held-out subsets of size $s$ are a simple random sample drawn without replacement, whose variance is at most that of the corresponding sample drawn with replacement (Theorem 4 of \cite{hoeffding1963probability}), namely $\sum_{S \notin \mathcal{T}^{(\ell)}, |S| = s} g(S)^2 / \pi^{(\ell)}_S$.
Since $\frac{\bar{m}}{k} \mathcal{D}(S) = m_s / (k \binom{n}{s})$ for $|S| = s$, and the split guarantees $t_s \geq m_s / (2k)$ for every size that is not enumerated (if $m_s \geq k$ then $t_s \geq \lfloor m_s / k \rfloor \geq m_s / (2k)$, and otherwise $t_s = 1 > m_s / (2k)$), we have $\pi^{(\ell)}_S \geq t_s / \binom{n}{s} \geq \frac12 \cdot \frac{\bar{m}}{k} \mathcal{D}(S)$.
Hence Equation \eqref{eq:estimate_var} below holds for the implementation with an extra factor of $2$ and $m$ replaced by $\bar{m}$, and the unpaired stratified design satisfies the analogous error bound with effective sample count $\bar{m}/2$ and weighted error restricted to the non-enumerated sizes.
Finally, the design loses nothing over a single sample: for a function $g$ that does not depend on $\ell$, the average over the $k$ folds of the estimates in the display above is exactly $\sum_{S \text{ drawn}} g(S) \binom{n}{|S|} / m_{|S|}$, the Horvitz--Thompson estimate from the whole sample, whenever every size is enumerated or drawn a multiple of $k$ times; and because the whole budget $m$ determines the per-size counts, the sizes that carry the most weight---the smallest and largest subsets---are enumerated as soon as their target count reaches $\binom{n}{s}$, which $k$ separate draws of $m/k$ subsets would not achieve.

Consider $\ell \in [k]$, and condition on its training data $\mathcal{T}^{(\ell)}$, so that $f^{(\ell)}$ is fixed. Expectations and variances in the following foldwise calculations are conditional on these training data.
For simplicity, suppose that $|\mathcal{S}^{(\ell)}| = m/k$.
We will first show that each estimated probabilistic value $\tilde{\phi}_i^{(\ell)}$ is unbiased:
\begin{align*}
    \E[\tilde{\phi}_i^{(\ell)}]
    &=\phi_i(f^{(\ell)}) 
    + \frac{k}{m} \E \left[\sum_{S' \in \mathcal{S}^{(\ell)}} 
    \sum_{S \subseteq [n]}
    \frac{\mathbbm{1}[S = S']}{\mathcal{D}(S)}
    [v(S) - f^{(\ell)}(S)] a_i(S)\right]
    \nonumber
    \\&= \phi_i(f^{(\ell)}) + \sum_{S \subseteq [n]}
    [v(S) - f^{(\ell)}(S)] a_i(S)
    \\&=\phi_i(f^{(\ell)})
    + \sum_{S \subseteq [n] \setminus \{i\}}
    [v(S \cup \{i\}) - v(S)] p_{|S|}
    - \sum_{S \subseteq [n] \setminus \{i\}}
    [f^{(\ell)}(S \cup \{i\}) - f^{(\ell)}(S)] p_{|S|}
    \\&=  \sum_{S \subseteq [n] \setminus \{i\}}
    [v(S \cup \{i\}) - v(S)] p_{|S|}
    = \phi_i.
\end{align*}
Since $\tilde{\phi}_i^{(\ell)}$ is unbiased, the final estimate $\E[\frac1{k}\sum_{\ell=1}^k \tilde{\phi}_i^{(\ell)}]$ is also unbiased by the linearity of expectation.
Next, we will analyze the variance of each estimate:
\begin{align}
    \Var[\tilde{\phi}_i^{(\ell)}]
    &= \frac{k^2}{m^2} \Var \left[\sum_{S' \in \mathcal{S}^{(\ell)}} 
    \sum_{S \subseteq [n]}
    \mathbbm{1}[S = S']
    [v(S) - f^{(\ell)}(S)] \frac{p_{|S|-1} \mathbbm{1}[i \in S] - p_{|S|} \mathbbm{1}[i \notin S]}{\mathcal{D}(S)}\right]
    \nonumber
    \\&\leq \frac{k}{m} \sum_{S \subseteq [n]}
    \mathcal{D}(S)
    [v(S) - f^{(\ell)}(S)]^2 \left(\frac{p_{|S|-1} \mathbbm{1}[i \in S] - p_{|S|} \mathbbm{1}[i \notin S]}{\mathcal{D}(S)}\right)^2
    \nonumber
    \\&= \frac{k}{m} \sum_{S \subseteq [n]}
    [v(S) - f^{(\ell)}(S)]^2 \frac{p_{|S|-1}^2 \mathbbm{1}[i \in S] + p_{|S|}^2 \mathbbm{1}[i \notin S]}{\mathcal{D}(S)}.
    \label{eq:estimate_var}
\end{align}
Define the training-data-measurable quantity
\begin{align*}
R_\ell := \sum_{S \subseteq [n]} [v(S)-f^{(\ell)}(S)]^2
\frac{p_{|S|-1}^2 |S|+p_{|S|}^2(n-|S|)}{\mathcal{D}(S)}.
\end{align*}
Summing Equation \eqref{eq:estimate_var} over coordinates gives
\begin{align}
\E\!\left[\|\bm{\phi}-\tilde{\bm{\phi}}^{(\ell)}\|_2^2
\,\middle|\,\mathcal{T}^{(\ell)}\right] \leq \frac{k}{m}R_\ell.
\label{eq:rv_var}
\end{align}
Apply Markov's inequality conditionally on this fold's training data, with failure probability $\delta/k$. For $R_\ell>0$,
\begin{align*}
\Pr\!\left(\|\bm{\phi}-\tilde{\bm{\phi}}^{(\ell)}\|_2^2
> \frac{k^2}{m\delta}R_\ell
\,\middle|\,\mathcal{T}^{(\ell)}\right) \leq \frac{\delta}{k}.
\end{align*}
If $R_\ell=0$, Equation \eqref{eq:rv_var} makes the error zero almost surely, so the same conclusion holds.
Averaging over the training data gives the same unconditional failure probability for each fold. A union bound, which does not require independence between the fold estimates, therefore gives, with probability at least $1-\delta$, all $k$ bounds simultaneously.
On that event, let $R_{\max}=\max_\ell R_\ell$, attained by the function $f_{\max}$ in the theorem. By the triangle inequality,
\begin{align*}
\|\bm{\phi}-\tilde{\bm{\phi}}\|_2
&\leq \frac1k\sum_{\ell=1}^k
\|\bm{\phi}-\tilde{\bm{\phi}}^{(\ell)}\|_2
\leq \sqrt{\frac{k^2}{m\delta}R_{\max}}.
\end{align*}
Squaring and taking $m\geq k^2 n/(\epsilon\delta)$ proves the stated bound. For fixed $k$, this is $O(n/(\epsilon\delta))$ samples.
We never condition jointly on all learned functions: $R_{\max}$ is introduced only after the foldwise probability bounds have been combined.
\end{proof}

\clearpage

\section{Regression MSR without Cross-Fitting}\label{appendix:all_samples}
With $k=1$, Algorithm \ref{alg:ours} trains one function $f$ on all $m$ samples and evaluates the estimator in Equation \eqref{eq:msr_intro} on those same samples; Section \ref{sec:regMSR} calls this the all-samples version (Linear MSR All and Tree MSR All below).
We show that, for Shapley values and the linear function class, the correction term in Equation \eqref{eq:msr_intro} is then identically zero, so the all-samples version returns the weighted least-squares solution itself: it is Leverage SHAP run on the same samples.

Write $\mathbf{x}_S \in \{0,1\}^n$ for the indicator vector of $S$, $\pi_S$ for the probability that $S$ appears in the sample (so $1/\pi_S$ plays the role of $1/(|\mathcal{S}|\,\mathcal{D}(S))$ in Equation \eqref{eq:msr_intro}), and $w_S = (p_{|S|-1} + p_{|S|})/\pi_S$.
Leverage SHAP fits $f(S) = v(\emptyset) + \langle \mathbf{x}_S, \hat{\boldsymbol{\phi}} \rangle$ by solving the sampled version of Equation \eqref{eq:linear_regression_connection},
\begin{align*}
\hat{\boldsymbol{\phi}} = \argmin_{\boldsymbol{\phi} : \langle \boldsymbol{\phi}, \mathbf{1} \rangle = v([n]) - v(\emptyset)} \; \sum_{S \in \mathcal{S}} w_S \big( v(S) - v(\emptyset) - \langle \mathbf{x}_S, \boldsymbol{\phi} \rangle \big)^2,
\end{align*}
and the Shapley values of this linear $f$ are $\phi_i(f) = \hat{\phi}_i$.
Writing $r_S = v(S) - f(S)$ for the residuals, the optimality conditions of the constrained problem state that the residuals are orthogonal to the regressors centered along the constraint direction: for every $i \in [n]$,
\begin{align}\label{eq:all_samples_normal}
\sum_{S \in \mathcal{S}} w_S\, r_S \left( x_{S,i} - \frac{|S|}{n} \right) = 0 .
\end{align}
(Parametrize $\boldsymbol{\phi} = \frac{v([n]) - v(\emptyset)}{n}\mathbf{1} + \boldsymbol{\psi}$ with $\langle \boldsymbol{\psi}, \mathbf{1} \rangle = 0$; the gradient with respect to $\boldsymbol{\psi}$ is the projection of $\sum_S w_S r_S \mathbf{x}_S$ onto the complement of $\mathbf{1}$, whose $i$th entry is the left-hand side above.)

The Shapley weights $p_s = \frac{s!(n-s-1)!}{n!}$ satisfy $s\, p_{s-1} = (n-s)\, p_s$ for $1 \le s \le n-1$, and this identity says exactly that the correction weights in Equation \eqref{eq:msr_intro} are proportional to the centered regressors:
\begin{align*}
p_{|S|-1} \mathbbm{1}[i \in S] - p_{|S|} \mathbbm{1}[i \notin S]
= (p_{|S|-1} + p_{|S|})\, x_{S,i} - p_{|S|}
= (p_{|S|-1} + p_{|S|}) \left( x_{S,i} - \frac{|S|}{n} \right).
\end{align*}
Dividing by $\pi_S$ and summing against the residuals, the correction term of the all-samples version is
\begin{align*}
\sum_{S \in \mathcal{S}} r_S\, \frac{p_{|S|-1} \mathbbm{1}[i \in S] - p_{|S|} \mathbbm{1}[i \notin S]}{\pi_S}
= \sum_{S \in \mathcal{S}} w_S\, r_S \left( x_{S,i} - \frac{|S|}{n} \right) = 0
\end{align*}
by Equation \eqref{eq:all_samples_normal}, so the all-samples version returns $\hat{\boldsymbol{\phi}}$.
(Our implementation fits $f$ without the intercept $v(\emptyset)$, which adds the constant $v(\emptyset)$ to every residual; the extra term $v(\emptyset) \sum_S w_S (x_{S,i} - |S|/n)$ also vanishes because we sample each subset together with its complement, and $S$ and $\bar{S}$ have equal weight $w_S$ and contributions $x_{S,i} - |S|/n$ and $(1 - x_{S,i}) - (n - |S|)/n$ that sum to zero.)
The identity $s\, p_{s-1} = (n-s)\, p_s$ holds only for the Shapley weights, so the all-samples version coincides with a regression estimator only in that case.
For Banzhaf values the analogous cancellation holds for Kernel Banzhaf \cite{liu2025kernelbanzhaffastrobust} when its least-squares fit is weighted by $1/\pi_S$, since the Banzhaf correction weights $2^{1-n}(2x_{S,i} - 1)/\pi_S$ are then proportional to the weighted regressors; our implementation fits Kernel Banzhaf without weights, and because the sampled subsets do not all share one inclusion probability (the per-size counts are rounded), a small correction remains.
For the tree function class there is no such identity: the correction is built from in-sample residuals, which a fitted tree has already shrunk, so the all-samples version is biased toward $\phi_i(f)$.

Table \ref{tab:all_samples} and Figure \ref{fig:all_samples} compare the all-samples version with the cross-fit estimators ($k=10$) on the experiments of Section \ref{sec:experiments}.
We pool the relative errors at $m=40n$ without noise over the datasets in each size group and the weightings in each family (10 runs per dataset and weighting), and report the ratio of the cross-fit estimator's average runtime to the all-samples version's.
For the linear function class the two versions are nearly interchangeable: the cross-fit estimator is more accurate in 53 of the 112 dataset--weighting pairs, and the two curves in Figure \ref{fig:all_samples} coincide in every panel.
(With Shapley weights, Linear MSR All is Leverage SHAP run on Linear MSR All's samples; the Leverage SHAP baseline in Section \ref{sec:experiments} draws a different sample, so the two are not numerically identical.)
For the tree function class, cross-fitting is more accurate in 90 of the 112 pairs and the all-samples version's error is typically $1.8\times$ larger (median over the pairs); the price of cross-fitting is a $2.6$--$9.9\times$ longer runtime, since it trains $k=10$ trees instead of one.

\begin{table}[tb!]
    \centering
    \caption{Mean relative error at $m=40n$ without noise, and the ratio of the cross-fit estimator's ($k=10$) runtime to the all-samples version's ($k=1$), pooled over the datasets in each size group and the weightings in each family (10 runs per dataset and weighting). The lower error of each pair is in bold. Symmetric Beta pools B(1,1), B(2,2), B(4,4), B(8,8); Asymmetric Beta pools B(1,1), B(1,2), B(1,4), B(1,8); Weighted Banzhaf pools WB(0.5) through WB(0.9).}
    \vspace{.5em}
    \sisetup{text-series-to-math=true}
\resizebox{\linewidth}{!}{
\begin{tabular}{llcccccc}
\toprule
 & & \multicolumn{4}{c}{Mean Relative Error ($m=40n$, no noise)} & \multicolumn{2}{c}{Runtime Ratio (Cross-fit / All-samples)} \\
\cmidrule(lr){3-6} \cmidrule(lr){7-8}
\textbf{Weighting Family} & \textbf{Datasets} & \textbf{Linear MSR} & \textbf{Linear MSR All} & \textbf{Tree MSR} & \textbf{Tree MSR All} & \textbf{Linear} & \textbf{Tree} \\
\midrule
Symmetric Beta & Small ($n<30$) & \num{8.43e-04} & \textbf{\num{7.96e-04}} & \textbf{\num{2.08e-04}} & \num{3.08e-04} & $2.1\times$ & $8.5\times$ \\
Symmetric Beta & Big ($n\geq 30$) & \textbf{\num{2.18e-04}} & \num{2.23e-04} & \textbf{\num{7.55e-05}} & \num{1.90e-04} & $2.3\times$ & $7.6\times$ \\
\midrule
Asymmetric Beta & Small ($n<30$) & \num{9.45e-04} & \textbf{\num{8.61e-04}} & \textbf{\num{1.26e-04}} & \num{1.88e-04} & $2.6\times$ & $8.7\times$ \\
Asymmetric Beta & Big ($n\geq 30$) & \textbf{\num{1.27e-03}} & \num{1.43e-03} & \textbf{\num{1.92e-04}} & \num{3.25e-04} & $2.5\times$ & $7.7\times$ \\
\midrule
Weighted Banzhaf & Small ($n<30$) & \num{1.51e-03} & \textbf{\num{1.31e-03}} & \textbf{\num{1.55e-04}} & \num{2.34e-04} & $2.5\times$ & $8.2\times$ \\
Weighted Banzhaf & Big ($n\geq 30$) & \textbf{\num{9.10e-04}} & \num{9.32e-04} & \textbf{\num{1.72e-04}} & \num{2.16e-04} & $1.9\times$ & $6.9\times$ \\
\midrule
Shapley & Small ($n<30$) & \num{7.72e-04} & \textbf{\num{4.83e-04}} & \textbf{\num{1.67e-04}} & \num{2.95e-04} & $1.4\times$ & $7.7\times$ \\
Shapley & Big ($n\geq 30$) & \textbf{\num{2.98e-04}} & \num{3.02e-04} & \textbf{\num{1.03e-04}} & \num{2.20e-04} & $3.7\times$ & $5.5\times$ \\
\midrule
Banzhaf & Small ($n<30$) & \num{8.49e-04} & \textbf{\num{8.08e-04}} & \textbf{\num{2.33e-04}} & \num{3.08e-04} & $1.2\times$ & $7.9\times$ \\
Banzhaf & Big ($n\geq 30$) & \num{1.49e-04} & \textbf{\num{1.41e-04}} & \textbf{\num{6.78e-05}} & \num{1.87e-04} & $1.4\times$ & $6.8\times$ \\
\bottomrule
\end{tabular}}

    \label{tab:all_samples}
\end{table}

\begin{figure}[tb!]
    \centering
    \includegraphics[width=\linewidth]{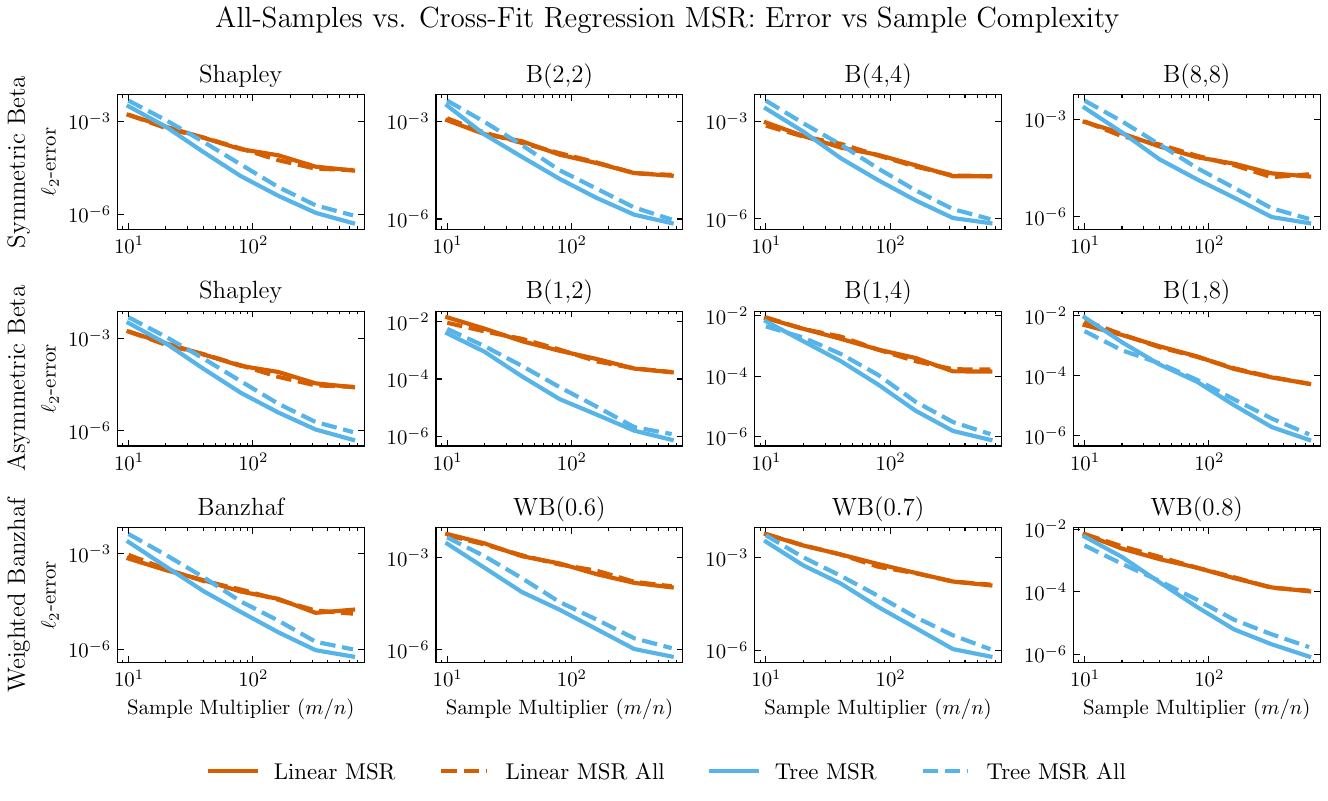}
    \caption{Error of the cross-fit estimators ($k=10$, solid) and the all-samples version ($k=1$, dashed) as a function of sample size on the large datasets ($n \geq 30$) without noise; each line is the mean over 10 runs, averaged over datasets. Linear MSR and its all-samples version coincide in every panel; Tree MSR lies below its all-samples version for most weightings and sample sizes.}
    \label{fig:all_samples}
\end{figure}

\section{Beta Shapley and Weighted Banzhaf Values}\label{appendix:more_probs}

Probabilistic values satisfy three intuitive properties:
\begin{itemize}
    \item Linearity: The probabilistic value of a linear combination of value functions is the linear combination of the probabilistic values for each value function i.e., $\phi_i(a v + b w) = a \phi_i(v) + b \phi_i(w)$ for real values $a, b \in \mathbb{R}$ and games $v,w: 2^{[n]} \to \mathbb{R}$.
    \item Null Player: The probabilistic value for a player that has no effect on any coalition is 0 i.e., if $v(S \cup \{i\}) = v(S)$ for all $S$ then $\phi_i = 0$.
    \item Symmetry: If two players contribute equally to all coalitions then they have the same probabilistic value i.e., if $v(S \cup \{i\}) = v(S \cup \{j\})$ for all $S$ then $\phi_i = \phi_j$.
\end{itemize}

In addition to these three properties, Shapley and Banzhaf values satisfy the efficiency and 2-efficiency properties, respectively:
\begin{itemize}
    \item Efficiency: The sum of probabilistic values is the difference between the whole coalition and the empty set i.e., $\sum_{i=1}^n \phi_i = v([n]) - v(\emptyset)$.
    Efficiency is desirable in settings where we want to attribute the value $v([n]) - v(\emptyset)$ to each player e.g., model prediction explanation or cost-sharing.
    \item 2-Efficiency: Let $v'$ be a game where $i$ and $j$ are combined into a single player $(i,j)$.
    That is, $v'$ is defined on $n-1$ players where the $i$ and $j$ players in $v$ are always considered together; effectively, $v'$ is only defined on subsets that contain both $i$ and $j$ or contain neither $i$ nor $j$.
    The 2-efficiency property requires that $\phi_i(v) + \phi_j(v) = \phi_{(i,j)}(v')$ for all $i,j$.
    2-Efficiency is desirable in settings where players can be combined into subgroups e.g., federated learning with client aggregation or games with alliances between players.
\end{itemize}

Shapley values are the most popular probabilistic value, in part because they are the only probabilistic value to satisfy the efficiency property \cite{shapley1951notes}.
The probabilistic weights for Shapley values are $p_{|S|} = \frac1{n} \binom{n-1}{|S|}^{-1}$, which weights all \textit{set sizes} equally.

The efficiency property is useful in settings where we want to allocate the total value of the game to the players. However, efficiency is not always appropriate especially when there are non-linear interactions between players that cannot be attributed to individuals. For these cases, a more appropriate property may be 2-efficiency which requires that probabilistic values add if players are merged.
The only probabilistic value that satisfies 2-efficiency is the Banzhaf value \cite{penrose1946elementary,banzhaf1964weighted}.
The probabilistic weight for Banzhaf values is $p_{|S|} = 1/2^{n-1}$, which weights all \textit{sets} equally.

\begin{figure}[ht]
    \centering
    \includegraphics[scale=.45]{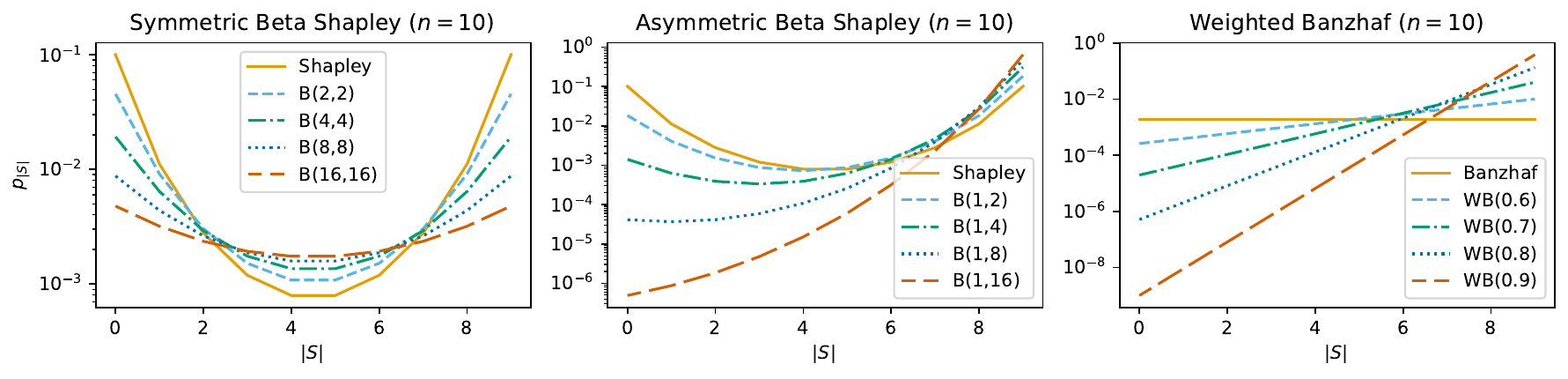}
    \caption{Probabilistic values by subset size for $n=10$. Beta Shapley values B$(\alpha, \beta)$ generalize Shapley values for $\alpha, \beta \in [1, \infty)$; increasing both $\alpha$ and $\beta$ flattens beta Shapley values while increasing just $\alpha$ (or just $\beta$) tilts beta Shapley values. Weighted Banzhaf values WB($p$) generalize Banzhaf values for $p \in (0,1)$; increasing (or decreasing) $p$ tilts weighted Banzhaf values.}
    \label{fig:prob_vals}
\end{figure}

Both Shapley and Banzhaf values have been generalized to beta Shapley values \cite{kwon2022beta} and weighted Banzhaf values \cite{li2024robust}, respectively.
Figure \ref{fig:prob_vals} plots $p_{|S|}$ for various beta Shapley and weighted Banzhaf values when $n=10$.
Beta Shapley values are defined by two parameters $\alpha, \beta \in [1, \infty)$. In particular, the probabilistic weight is
\begin{align*}
    \frac{\text{Beta}(|S| + \beta, n - |S| - 1 + \alpha)}{\text{Beta}(\alpha, \beta)}.
\end{align*}
Setting $\alpha=\beta=1$ recovers Shapley values.
Weighted Banzhaf values are defined by one parameter $p \in (0,1)$. In particular, the probabilistic weight is
\begin{align*}
    p^{|S|}(1-p)^{n-|S|-1}.
\end{align*}
Setting $p=\frac12$ recovers Banzhaf values.

\clearpage
\section{Computing Probabilistic Values of Tree-based Models}\label{appendix:tree_prob}

In this section, we show how to efficiently compute the probabilistic values of a tree when the value function is interventional feature attribution.
Unfortunately, we cannot directly generalize analogous algorithms for Shapley and Banzhaf values \cite{lundberg2017unified,lundberg2020local,karczmarz2022improved}, since they uses a property which does not hold for all probabilistic values.
In particular, prior approaches assume that if a value function only has contributions from $n' < n$ players, the Shapley/Banzhaf value on the induced game of those $n'$ players is the same as the Shapley/Banzhaf value on the original value function with all $n$ players. Our approach is based on an alternative way of viewing tree-based models that avoids the need for this property.

Consider the value function $v:2^{[n]} \to \mathbb{R}$ induced by a tree with explicand $\mathbf{x}^e$ and baseline $\mathbf{x}^b$.
We decompose $v$ into a sum of \emph{path value functions} $\bigl\{v^{P}\bigr\}_{P}$, where each $v^{P}: 2^{[n]} \to \mathbb{R}$ corresponds to a distinct root-to-leaf path $P$.
In particular,
$$
    v(S) 
    \;=\; 
    \sum_{P} v^{P}(S),
    \qquad
    \text{where}
    \quad
    v^{P}(S) 
    = 
    \begin{cases}
        \text{leaf value of }P & \text{if } S \text{ follows path } P \text{ on } \mathbf{x}^e, \mathbf{x}^b,\\
        0 & \text{otherwise}.
    \end{cases}
$$

By the linearity property of probabilistic values, the contribution of feature $i$ to the full tree model $v$ can be expressed as the sum of its contributions to each path model $v^P$. Specifically,
$$
    \phi_i\bigl(v\bigr)
    \;=\;
    \phi_i\Bigl(\,\sum_{P} v^{P}\Bigr)
    \;=\;
    \sum_{P} \phi_i\bigl(v^{P}\bigr).
$$
Therefore, it suffices to compute $\phi_i(v^P)$ for each path model $v^P$ and aggregate their contributions over all paths. 
To this end, we will introduce the following notation.
Given probabilistic weights $\mathbf{p} = [p_0, \ldots, p_{n-1}]\in [0,1]^n$, for each path $P$ reachable by some hybrid of $\mathbf{x}^e$ and $\mathbf{x}^b$, define
\begin{itemize}
    \item $S_{P}$ as the set of features in $P$ whose conditions are satisfied by $\mathbf{x}^e$ but not by $\mathbf{x}^b$,
    \item $N_{P}$ as the set of features in $P$ whose conditions are satisfied by $\mathbf{x}^b$ but not by $\mathbf{x}^e$,
    \item $\ell_P$ as the final leaf value on path $P$.
\end{itemize}

Recall we can write the $i$th probabilistic value as
\begin{align*}
    \phi_i (v)=\sum_{S \subseteq [n]: i \in S} p_{|S|-1} v(S) 
    - \sum_{S \subseteq [n]: i \notin S} p_{|S|} v(S).
\end{align*}
Using the definition of $v^P$, $S_P$, and $N_P$, we can consider $\phi_i(v^P)$ in three cases:

\begin{itemize}

    \item \textbf{Case 1: $i \in S_P$:} We need $i \in S$ in order to reach the leaf i.e., $v^P(S) = 0$ unless $S_P \subseteq S \subseteq [n] \setminus N_P$ and $i \in S$. Then,
    $$
        \phi_i(v^P) = \sum_{S_P \subseteq S \subseteq [n] \setminus N_P} p_{|S|-1} \cdot \ell_P = \sum_{l=|S_P|}^{n-|N_P|} p_{l-1} \binom{n-|N_P|-|S_P|}{l-|S_P|}  \cdot \ell_P
    $$

    \item \textbf{Case 2: $i \in N_P$:} We need $i \notin S$ in order to reach the leaf i.e., $v^P(S) = 0$ unless $S_P \subseteq S \subseteq [n] \setminus N_P$ and $i \notin S$. Then,
    $$
        \phi_i(v^P) = - \sum_{l=|S_P|}^{n-|N_P|} p_{l} \binom{n-|N_P|-|S_P|}{l-|S_P|} \cdot \ell_P
    $$

    \item \textbf{Case 3: $i \notin N_P$ and $i \notin S_P$:} We reach the leaf whether $i \in S$ or not i.e., $v^P(S) = 0$ unless $S_P \subseteq S \subseteq [n] \setminus N_P$. Then,
    \begin{align*}
        \phi_i(v^P) = &\sum_{l=|S_P|+1}^{n-|N_P|} p_{l-1} \binom{n-|N_P|-|S_P|-1}{l-|S_P|-1} \ell_P \\
        &- \sum_{l=|S_P|}^{n-|N_P|-1} p_{l} \binom{n-|N_P|-|S_P|-1}{l-|S_P|} \ell_P = 0
    \end{align*}

\end{itemize}

\subsection{TreeProb Pseudocode}

Algorithm~\ref{alg:treeprob} efficiently explores all root-to-leaf paths, maintaining counters for how many times each feature has been "seen" under the explicand ($\text{ef\_seen}$) or the baseline ($\text{bf\_seen}$). 
When a branching feature is encountered for the first time on a path, the algorithm branches into two recursive calls—one following $\mathbf{x}^e$, the other $\mathbf{x}^b$—and updates the \emph{cardinalities} $s_P := |S_P|$ or $n_P := |N_P|$ accordingly, depending on which feature value is consistent with the split.
At each leaf node, the algorithm computes the contribution based on the current $(s_P, n_P)$ and aggregates these into the overall attribution vector $\boldsymbol{\phi}$.
Algorithm~\ref{alg:treeprob} preserves the original complexity and traversal logic of Tree SHAP, while generalizing the feature-contribution calculation to the probabilistic formulation described above, making it applicable to any probabilistic value.

\begin{algorithm}[ht]
\caption{TreeProb with Interventional Feature Perturbation}
\label{alg:treeprob}
\begin{algorithmic}[1]
\STATE {\bfseries Input:} tree ensemble $\mathcal{T}$; explicand $\mathbf{x}^e \in \mathbb{R}^n$; set of baselines $\mathcal{B} \subset \mathbb{R}^n$; probabilistic weights $\mathbf{p} \in [0,1]^n$
\STATE {\bfseries Output:} Exact probabilistic values $\phi_1, \ldots, \phi_n$ of the interventional value function $v(S) = \frac{1}{|\mathcal{B}|}\sum_{\mathbf{x}^b \in \mathcal{B}} \mathcal{T}(\mathbf{x}^e_S, \mathbf{x}^b_{[n]\setminus S})$
\vspace{6pt}

\STATE \textbf{function} {\textsc{Recurse}$(\text{node}, s_P, n_P, \text{ef\_seen}, \text{bf\_seen})$}
    \IF{\text{node is a leaf}}
        \STATE $\text{pos\_term} \gets \text{node.value} \cdot \displaystyle\sum_{l=s_P}^{\,n-n_P} p_{l-1}\binom{n - n_P - s_P}{\,l - s_P\,}$
        \STATE $\text{neg\_term} \gets -\,\text{node.value} \cdot \displaystyle\sum_{l=s_P}^{\,n-n_P} p_{l}\binom{n - n_P - s_P}{\,l - s_P\,}$
        \STATE \textbf{return} $(\text{pos\_term},\; \text{neg\_term})$
    \ENDIF

    \STATE $x_e\_\text{child} \gets 
        \begin{cases}
            \text{node.leftchild} & \text{if } x_e[\text{node.feat}] < \text{node.t} \\
            \text{node.rightchild} & \text{otherwise}
        \end{cases}$
    \STATE $x_b\_\text{child} \gets 
        \begin{cases}
            \text{node.leftchild} & \text{if } x_b[\text{node.feat}] < \text{node.t} \\
            \text{node.rightchild} & \text{otherwise}
        \end{cases}$

    \IF{$\text{ef\_seen}[\text{node.feat}] > 0$}
        \STATE \textbf{return} $\textsc{Recurse}(x_e\_\text{child},\, s_P,\, n_P,\, \text{ef\_seen},\, \text{bf\_seen})$
    \ENDIF
    \IF{$\text{bf\_seen}[\text{node.feat}] > 0$}
        \STATE \textbf{return} $\textsc{Recurse}(x_b\_\text{child},\, s_P,\, n_P,\, \text{ef\_seen},\, \text{bf\_seen})$
    \ENDIF
    \IF{$x_e\_\text{child} = x_b\_\text{child}$}
        \STATE \textbf{return} $\textsc{Recurse}(x_e\_\text{child},\, s_P,\, n_P,\, \text{ef\_seen},\, \text{bf\_seen})$
    \ELSE
        \STATE $\text{ef\_seen}[\text{node.feat}] \gets \text{ef\_seen}[\text{node.feat}] + 1$
        \STATE $(\text{pos}_e,\, \text{neg}_e) \gets \textsc{Recurse}(x_e\_\text{child},\, s_P + 1,\, n_P,\, \text{ef\_seen},\, \text{bf\_seen})$
        \STATE $\text{ef\_seen}[\text{node.feat}] \gets \text{ef\_seen}[\text{node.feat}] - 1$

        \STATE $\text{bf\_seen}[\text{node.feat}] \gets \text{bf\_seen}[\text{node.feat}] + 1$
        \STATE $(\text{pos}_b,\, \text{neg}_b) \gets \textsc{Recurse}(x_b\_\text{child},\, s_P,\, n_P + 1,\, \text{ef\_seen},\, \text{bf\_seen})$
        \STATE $\text{bf\_seen}[\text{node.feat}] \gets \text{bf\_seen}[\text{node.feat}] - 1$

        \STATE $\phi_{\text{temp}}[\text{node.feat}] \gets \phi_{\text{temp}}[\text{node.feat}] + (\text{pos}_e + \text{neg}_b)$
        \STATE \textbf{return} $(\text{pos}_e + \text{pos}_b,\; \text{neg}_e + \text{neg}_b)$
    \ENDIF
\STATE \textbf{end function}
\vspace{6pt}

\STATE \textbf{Initialize} $\boldsymbol{\phi} \gets \mathbf{0}^{n}$
\FOR{each tree $t$ in $\mathcal{T}$}
    \FOR{each baseline $\mathbf{x}^b \in \mathcal{B}$}
        \STATE $\phi_{\text{temp}} \gets \mathbf{0}^{n}$
        \STATE $\textsc{Recurse}(t.\text{root},\, 0,\, 0,\, \mathbf{0}^{n},\, \mathbf{0}^{n})$
        \STATE $\boldsymbol{\phi} \gets \boldsymbol{\phi} + \phi_{\text{temp}}$
    \ENDFOR
\ENDFOR
\STATE $\boldsymbol{\phi} \gets \boldsymbol{\phi} / |\mathcal{B}|$
\STATE \textbf{if} $\mathcal{T}$ averages its trees (e.g., a random forest) \textbf{then} $\boldsymbol{\phi} \gets \boldsymbol{\phi} / |\mathcal{T}|$ \hfill $\triangleright$ Boosted ensembles sum their trees, so no division.
\STATE \textbf{return} $\boldsymbol{\phi}$
\end{algorithmic}
\end{algorithm}

\clearpage
\section{Baselines}\label{appendix:baselines}

In this section, we describe the baselines we compare against for estimating probabilistic values.
These baselines broadly fall into three categories: standard Monte Carlo methods, maximum sample reuse methods, and regression methods.

For a more technical description of many of these baselines, please refer to Appendix D of \cite{li2024one}.

\textbf{Monte Carlo Methods}
The standard Monte Carlo estimator estimates each probabilistic value individually by sampling each term in the summation with probability proportional to its weight.

\emph{Weighted Sampling Lift (WSL)} \cite{kwon2022beta} is the standard Monte Carlo, but where subsets are sampled according to the Shapley weights and reweighted to produce unbiased estimates of the probabilistic values.

\emph{Permutation SHAP} \cite{castro2009polynomial} is similar to Monte Carlo estimates except that subsets are sampled in permutations; that is, the first element in the sampled permutation is one sampled subset, the first two elements are another sampled subset, and so on. Because each set \textit{size} is weighted equally by Shapley values, sampling permutations without any reweighting gives estimates that are the Shapley values in expectation. Permutation SHAP gives close to state-of-the-art performance.

\emph{Weighted SHAP} estimator \cite{kwon2022weightedshap} generalizes permutation sampling approach to probabilistic value.
Random permutations are drawn as before but now reweighted by the probabilistic value weights so that the final estimates are unbiased.

\textbf{Maximum Sample Reuse Methods}
Monte Carlo methods are unbiased but inefficient in the sense that they only use each sample to compute the estimates of one or two probabilistic values.
The key observation of Maximum Sample Reuse (MSR) estimators is that each probabilistic value $\phi_i$ can be written as two summations, one over sets that include $i$ and one over sets that do not.
Then the MSR methods use each sample to update every summation.
First used for Banzhaf values \cite{wang2023databanzhaf}, MSR has since been generalized to other probabilistic values with different sampling distributions.

\emph{Approximation without Requesting Marginals (ARM)} \cite{kolpaczki2024approximating} is a kind of MSR estimator.
Half the samples are drawn with probability $p_{|S|-1}$ while the other half are drawn with probability $p_{|S|}$.
In order to avoid the numerical instability of reweighting, the final estimate only includes the first half of samples if $i\in S$ and the second if $i \notin S$.

\emph{One sample Fits All (OFA)} \cite{li2024one} similarly uses maximum sample reuse but samples according to a more complicated distribution.

\textbf{Regression Methods}
A parallel line of work fits linear models $f$ to $v$ then returns the probabilistic values of $f$.
These approaches originate for estimating Shapley values, and are based on a linear regression problem that exactly recovers the Shapley values when solved exactly \cite{charnes1988extremal}.

\emph{Kernel SHAP} \cite{lundberg2017unified} samples subsets from this regression problem with probability proportional to their weighting in the regression problem.

\emph{Leverage SHAP} \cite{musco2024provably} similarly samples subsets but with probability proportional to their \textit{statistical leverage} in the regression problem, resulting in state-of-the-art Shapley value estimates and error bounds that depend on the fit of $f$ to $v$.

\emph{Kernel Banzhaf} \cite{liu2025kernelbanzhaffastrobust} is similar to Kernel SHAP and Leverage SHAP but estimates Banzhaf values, and is based on a regression formulation specific to Banzhaf values \cite{hammer1992approximations}.

There is a known generalization of the Shapley value regression problem \cite{ruiz1998family}, which, when solved exactly, recovers probabilistic values up to additive constant. Since Shapley values satisfy efficiency, this constant is efficient to exactly compute for Shapley values. However, for general probabilistic values, the constant depends on the entire value function $v$ and must be estimated, introducing another source of error.

The \emph{Generic Estimator based on Least Squares (GELS)} \cite{li2024faster} estimates the constant by adding a dummy variable with probabilistic value 0.
Instead of fitting a linear function $f$, GELS considers the closed-form solution to the regression problem and effectively applies a maximum sample reuse estimator to the underlying matrix-vector multiplication.
The final estimates are adjusted by subtracting the value of the dummy variable.

The \emph{Average Marginal Effect (AME)} \cite{lin2022measuring} is another regression estimator that uses a different regression formulation. For probabilistic values that satisfy a specific condition, the probabilistic values can be written as an infinitely tall regression problem.
The estimator samples this regression problem and solves the approximate version.

\clearpage

\section{Experiments on Small Datasets with Neural Network Models}\label{appendix:net_experiments}

\begin{table}[ht]
    \centering
    \caption{Summary statistics of the relative squared $\ell_2$ error between estimated and true probabilistic values when $m=40n$. We summarize the error over small datasets $(n < 30)$, for which the probabilistic values of a neural network model can be feasibly computed. On average over all probabilistic values, Tree MSR produces estimates with mean error that is $243\times$ lower than the best estimator from prior work.}
    \vspace{.5em}
    \resizebox{\linewidth}{!}{ 
\begin{tabular} {lcccccccccccc|c}
\hline
 & \textbf{B(1,1)} & \textbf{B(2,2)} & \textbf{B(4,4)} & \textbf{B(8,8)} & \textbf{B(1,2)} & \textbf{B(1,4)} & \textbf{B(1,8)} & \textbf{WB(0.5)} & \textbf{WB(0.6)} & \textbf{WB(0.7)} & \textbf{WB(0.8)} & \textbf{WB(0.9)} & \textbf{Mean} \\ 
\hline 
\textbf{LinearMSR} &  &  &  &  &  &  &  &  &  &  &  &  \\ 
\hspace{7pt}Mean & \cellcolor{silver!60}$\num{7.34e-04}$ & \cellcolor{silver!60}$\num{8.17e-04}$ & \cellcolor{silver!60}$\num{8.53e-04}$ & \cellcolor{silver!60}$\num{9.66e-04}$ & \cellcolor{silver!60}$\num{1.87e-03}$ & \cellcolor{silver!60}$\num{9.30e-04}$ & \cellcolor{silver!60}$\num{2.46e-04}$ & \cellcolor{silver!60}$\num{7.04e-04}$ & \cellcolor{silver!60}$\num{2.69e-03}$ & \cellcolor{silver!60}$\num{2.81e-03}$ & \cellcolor{silver!60}$\num{1.20e-03}$ & \cellcolor{silver!60}$\num{1.37e-04}$ & \cellcolor{silver!60}$\num{1.16e-03}$ \\ 
\hspace{7pt}1st Quartile & \cellcolor{silver!60}$\num{1.55e-06}$ & \cellcolor{silver!60}$\num{1.30e-06}$ & \cellcolor{silver!60}$\num{1.23e-06}$ & \cellcolor{silver!60}$\num{6.17e-07}$ & \cellcolor{silver!60}$\num{4.21e-05}$ & \cellcolor{silver!60}$\num{2.66e-05}$ & \cellcolor{silver!60}$\num{4.04e-06}$ & \cellcolor{silver!60}$\num{2.09e-06}$ & \cellcolor{silver!60}$\num{4.71e-05}$ & \cellcolor{silver!60}$\num{2.34e-05}$ & \cellcolor{silver!60}$\num{2.14e-05}$ & \cellcolor{silver!60}$\num{1.96e-06}$ & \cellcolor{silver!60}$\num{1.45e-05}$ \\ 
\hspace{7pt}2nd Quartile & \cellcolor{silver!60}$\num{6.80e-05}$ & \cellcolor{silver!60}$\num{7.52e-05}$ & \cellcolor{silver!60}$\num{1.01e-04}$ & \cellcolor{silver!60}$\num{6.54e-05}$ & \cellcolor{silver!60}$\num{1.01e-03}$ & \cellcolor{silver!60}$\num{5.08e-04}$ & \cellcolor{silver!60}$\num{1.31e-04}$ & \cellcolor{silver!60}$\num{1.08e-04}$ & \cellcolor{silver!60}$\num{1.70e-03}$ & \cellcolor{silver!60}$\num{1.16e-03}$ & \cellcolor{silver!60}$\num{5.98e-04}$ & \cellcolor{silver!60}$\num{4.75e-05}$ & \cellcolor{silver!60}$\num{4.64e-04}$ \\ 
\hspace{7pt}3rd Quartile & \cellcolor{silver!60}$\num{8.96e-04}$ & \cellcolor{silver!60}$\num{9.48e-04}$ & \cellcolor{silver!60}$\num{9.54e-04}$ & \cellcolor{silver!60}$\num{1.12e-03}$ & \cellcolor{silver!60}$\num{3.14e-03}$ & \cellcolor{silver!60}$\num{1.32e-03}$ & \cellcolor{silver!60}$\num{3.55e-04}$ & \cellcolor{silver!60}$\num{6.18e-04}$ & \cellcolor{silver!60}$\num{5.23e-03}$ & \cellcolor{silver!60}$\num{4.46e-03}$ & \cellcolor{silver!60}$\num{1.74e-03}$ & \cellcolor{silver!60}$\num{1.39e-04}$ & \cellcolor{silver!60}$\num{1.74e-03}$ \\ 
\hline 
\textbf{TreeMSR} &  &  &  &  &  &  &  &  &  &  &  &  \\ 
\hspace{7pt}Mean & \cellcolor{gold!60}$\num{1.75e-04}$ & \cellcolor{gold!60}$\num{1.89e-04}$ & \cellcolor{gold!60}$\num{1.88e-04}$ & \cellcolor{gold!60}$\num{2.78e-04}$ & \cellcolor{gold!60}$\num{1.73e-04}$ & \cellcolor{gold!60}$\num{1.17e-04}$ & \cellcolor{gold!60}$\num{3.84e-05}$ & \cellcolor{gold!60}$\num{2.13e-04}$ & \cellcolor{gold!60}$\num{2.27e-04}$ & \cellcolor{gold!60}$\num{1.99e-04}$ & \cellcolor{gold!60}$\num{1.16e-04}$ & \cellcolor{gold!60}$\num{1.81e-05}$ & \cellcolor{gold!60}$\num{1.61e-04}$ \\ 
\hspace{7pt}1st Quartile & \cellcolor{gold!60}$\num{7.75e-08}$ & \cellcolor{gold!60}$\num{1.49e-07}$ & \cellcolor{gold!60}$\num{1.97e-07}$ & \cellcolor{gold!60}$\num{8.39e-08}$ & \cellcolor{gold!60}$\num{8.65e-08}$ & \cellcolor{gold!60}$\num{4.11e-08}$ & \cellcolor{gold!60}$\num{1.07e-08}$ & \cellcolor{gold!60}$\num{2.23e-07}$ & \cellcolor{gold!60}$\num{1.15e-07}$ & \cellcolor{gold!60}$\num{1.13e-07}$ & \cellcolor{gold!60}$\num{4.96e-08}$ & \cellcolor{gold!60}$\num{4.98e-09}$ & \cellcolor{gold!60}$\num{9.60e-08}$ \\ 
\hspace{7pt}2nd Quartile & \cellcolor{gold!60}$\num{1.74e-05}$ & \cellcolor{gold!60}$\num{2.22e-05}$ & \cellcolor{gold!60}$\num{1.85e-05}$ & \cellcolor{gold!60}$\num{1.43e-05}$ & \cellcolor{gold!60}$\num{1.62e-05}$ & \cellcolor{gold!60}$\num{6.35e-06}$ & \cellcolor{gold!60}$\num{1.31e-06}$ & \cellcolor{gold!60}$\num{2.54e-05}$ & \cellcolor{gold!60}$\num{1.74e-05}$ & \cellcolor{gold!60}$\num{1.49e-05}$ & \cellcolor{gold!60}$\num{8.97e-06}$ & \cellcolor{gold!60}$\num{2.65e-07}$ & \cellcolor{gold!60}$\num{1.36e-05}$ \\ 
\hspace{7pt}3rd Quartile & \cellcolor{gold!60}$\num{1.01e-04}$ & \cellcolor{gold!60}$\num{9.74e-05}$ & \cellcolor{gold!60}$\num{9.40e-05}$ & \cellcolor{gold!60}$\num{5.65e-05}$ & \cellcolor{gold!60}$\num{6.30e-05}$ & \cellcolor{gold!60}$\num{4.88e-05}$ & \cellcolor{gold!60}$\num{1.10e-05}$ & \cellcolor{gold!60}$\num{6.85e-05}$ & \cellcolor{gold!60}$\num{7.74e-05}$ & \cellcolor{gold!60}$\num{7.64e-05}$ & \cellcolor{gold!60}$\num{3.74e-05}$ & \cellcolor{gold!60}$\num{1.80e-06}$ & \cellcolor{gold!60}$\num{6.11e-05}$ \\ 
\hline 
\textbf{OFA} &  &  &  &  &  &  &  &  &  &  &  &  \\ 
\hspace{7pt}Mean & $\num{3.15e-02}$ & $\num{3.60e-02}$ & $\num{4.39e-02}$ & $\num{4.50e-02}$ & $\num{3.03e-02}$ & \cellcolor{bronze!60}$\num{2.48e-02}$ & \cellcolor{bronze!60}$\num{1.93e-02}$ & \cellcolor{bronze!60}$\num{5.01e-02}$ & $\num{9.12e-02}$ & $\num{1.65e-01}$ & \cellcolor{bronze!60}$\num{4.21e-02}$ & $\num{5.71e-02}$ & $\num{5.30e-02}$ \\ 
\hspace{7pt}1st Quartile & $\num{2.47e-02}$ & $\num{2.93e-02}$ & $\num{2.70e-02}$ & $\num{3.34e-02}$ & $\num{1.92e-02}$ & $\num{1.42e-02}$ & \cellcolor{bronze!60}$\num{1.14e-02}$ & $\num{3.76e-02}$ & $\num{3.65e-02}$ & $\num{3.37e-02}$ & $\num{3.20e-02}$ & \cellcolor{bronze!60}$\num{1.61e-02}$ & $\num{2.63e-02}$ \\ 
\hspace{7pt}2nd Quartile & $\num{3.11e-02}$ & $\num{3.41e-02}$ & $\num{4.61e-02}$ & $\num{4.25e-02}$ & $\num{2.85e-02}$ & \cellcolor{bronze!60}$\num{2.18e-02}$ & \cellcolor{bronze!60}$\num{1.55e-02}$ & $\num{4.94e-02}$ & $\num{5.09e-02}$ & $\num{4.34e-02}$ & $\num{3.95e-02}$ & \cellcolor{bronze!60}$\num{2.56e-02}$ & $\num{3.57e-02}$ \\ 
\hspace{7pt}3rd Quartile & $\num{3.95e-02}$ & $\num{4.44e-02}$ & $\num{5.50e-02}$ & $\num{5.13e-02}$ & $\num{3.91e-02}$ & \cellcolor{bronze!60}$\num{3.17e-02}$ & \cellcolor{bronze!60}$\num{1.98e-02}$ & \cellcolor{bronze!60}$\num{6.28e-02}$ & $\num{8.17e-02}$ & $\num{6.33e-02}$ & \cellcolor{bronze!60}$\num{5.29e-02}$ & \cellcolor{bronze!60}$\num{3.54e-02}$ & \cellcolor{bronze!60}$\num{4.81e-02}$ \\ 
\hline 
\textbf{WSL} &  &  &  &  &  &  &  &  &  &  &  &  \\ 
\hspace{7pt}Mean & $\num{1.06e-02}$ & $\num{1.92e-02}$ & $\num{3.82e-02}$ & $\num{6.10e-02}$ & $\num{4.47e-02}$ & $\num{6.12e-02}$ & $\num{1.17e-01}$ & $\num{8.63e-02}$ & $\num{9.37e-02}$ & $\num{1.31e-01}$ & $\num{1.36e-01}$ & $\num{1.02e-01}$ & $\num{7.50e-02}$ \\ 
\hspace{7pt}1st Quartile & $\num{4.95e-05}$ & $\num{4.63e-03}$ & $\num{1.21e-02}$ & $\num{1.64e-02}$ & $\num{1.08e-02}$ & $\num{1.50e-02}$ & $\num{3.51e-02}$ & $\num{1.62e-02}$ & $\num{1.60e-02}$ & $\num{2.40e-02}$ & $\num{3.35e-02}$ & $\num{2.91e-02}$ & $\num{1.77e-02}$ \\ 
\hspace{7pt}2nd Quartile & $\num{4.93e-03}$ & $\num{1.16e-02}$ & $\num{2.60e-02}$ & $\num{3.00e-02}$ & $\num{2.47e-02}$ & $\num{4.20e-02}$ & $\num{8.53e-02}$ & $\num{5.27e-02}$ & $\num{6.48e-02}$ & $\num{6.22e-02}$ & $\num{8.66e-02}$ & $\num{5.61e-02}$ & $\num{4.56e-02}$ \\ 
\hspace{7pt}3rd Quartile & $\num{1.01e-02}$ & $\num{2.43e-02}$ & $\num{4.35e-02}$ & $\num{5.85e-02}$ & $\num{3.77e-02}$ & $\num{7.43e-02}$ & $\num{1.41e-01}$ & $\num{1.13e-01}$ & $\num{1.54e-01}$ & $\num{1.57e-01}$ & $\num{1.80e-01}$ & $\num{1.42e-01}$ & $\num{9.46e-02}$ \\ 
\hline 
\textbf{GELS} &  &  &  &  &  &  &  &  &  &  &  &  \\ 
\hspace{7pt}Mean & $\num{1.50e-01}$ & $\num{1.15e-01}$ & $\num{1.25e-01}$ & $\num{8.07e-02}$ & $\num{9.31e-02}$ & $\num{7.18e-02}$ & $\num{6.49e-02}$ & $\num{1.00e-01}$ & $\num{6.56e-02}$ & $\num{5.52e-02}$ & $\num{5.06e-02}$ & $\num{4.81e-02}$ & $\num{8.50e-02}$ \\ 
\hspace{7pt}1st Quartile & $\num{8.29e-02}$ & $\num{4.95e-02}$ & $\num{6.11e-02}$ & $\num{4.73e-02}$ & $\num{4.18e-02}$ & $\num{3.96e-02}$ & $\num{3.96e-02}$ & $\num{5.25e-02}$ & $\num{3.14e-02}$ & $\num{3.07e-02}$ & $\num{2.71e-02}$ & $\num{3.06e-02}$ & $\num{4.45e-02}$ \\ 
\hspace{7pt}2nd Quartile & $\num{1.22e-01}$ & $\num{7.35e-02}$ & $\num{1.03e-01}$ & $\num{7.04e-02}$ & $\num{7.02e-02}$ & $\num{6.53e-02}$ & $\num{5.55e-02}$ & $\num{8.51e-02}$ & $\num{4.29e-02}$ & $\num{3.78e-02}$ & $\num{4.28e-02}$ & $\num{4.25e-02}$ & $\num{6.76e-02}$ \\ 
\hspace{7pt}3rd Quartile & $\num{1.76e-01}$ & $\num{1.36e-01}$ & $\num{1.88e-01}$ & $\num{9.45e-02}$ & $\num{1.14e-01}$ & $\num{9.04e-02}$ & $\num{8.23e-02}$ & $\num{1.21e-01}$ & $\num{5.82e-02}$ & \cellcolor{bronze!60}$\num{6.04e-02}$ & $\num{6.58e-02}$ & $\num{5.88e-02}$ & $\num{1.04e-01}$ \\ 
\hline 
\textbf{ARM} &  &  &  &  &  &  &  &  &  &  &  &  \\ 
\hspace{7pt}Mean & $\num{8.25e-02}$ & $\num{6.22e-02}$ & $\num{5.22e-02}$ & $\num{6.50e-02}$ & $\num{8.07e-02}$ & $\num{5.89e-02}$ & $\num{4.88e-02}$ & $\num{5.68e-02}$ & $\num{4.91e-02}$ & \cellcolor{bronze!60}$\num{4.95e-02}$ & $\num{4.68e-02}$ & \cellcolor{bronze!60}$\num{3.90e-02}$ & $\num{5.76e-02}$ \\ 
\hspace{7pt}1st Quartile & $\num{4.93e-02}$ & $\num{3.38e-02}$ & $\num{3.07e-02}$ & $\num{3.06e-02}$ & $\num{3.79e-02}$ & $\num{3.82e-02}$ & $\num{3.03e-02}$ & $\num{2.51e-02}$ & $\num{3.04e-02}$ & $\num{3.03e-02}$ & $\num{2.86e-02}$ & $\num{2.15e-02}$ & $\num{3.22e-02}$ \\ 
\hspace{7pt}2nd Quartile & $\num{6.05e-02}$ & $\num{4.44e-02}$ & $\num{3.93e-02}$ & $\num{4.84e-02}$ & $\num{6.49e-02}$ & $\num{5.68e-02}$ & $\num{4.46e-02}$ & $\num{3.93e-02}$ & $\num{4.24e-02}$ & $\num{4.38e-02}$ & $\num{4.09e-02}$ & $\num{3.32e-02}$ & $\num{4.66e-02}$ \\ 
\hspace{7pt}3rd Quartile & $\num{9.64e-02}$ & $\num{8.92e-02}$ & $\num{5.95e-02}$ & $\num{7.79e-02}$ & $\num{1.09e-01}$ & $\num{7.20e-02}$ & $\num{6.53e-02}$ & $\num{6.78e-02}$ & $\num{5.99e-02}$ & $\num{6.48e-02}$ & $\num{5.54e-02}$ & $\num{4.93e-02}$ & $\num{7.22e-02}$ \\ 
\hline 
\textbf{WeightedSHAP} &  &  &  &  &  &  &  &  &  &  &  &  \\ 
\hspace{7pt}Mean & \cellcolor{bronze!60}$\num{6.05e-03}$ & \cellcolor{bronze!60}$\num{8.75e-03}$ & \cellcolor{bronze!60}$\num{1.98e-02}$ & \cellcolor{bronze!60}$\num{2.88e-02}$ & \cellcolor{bronze!60}$\num{1.42e-02}$ & $\num{4.23e-02}$ & $\num{5.43e-02}$ & $\num{5.21e-02}$ & \cellcolor{bronze!60}$\num{3.15e-02}$ & $\num{5.52e-02}$ & $\num{4.86e-02}$ & $\num{1.08e-01}$ & \cellcolor{bronze!60}$\num{3.91e-02}$ \\ 
\hspace{7pt}1st Quartile & \cellcolor{bronze!60}$\num{2.20e-05}$ & \cellcolor{bronze!60}$\num{1.73e-03}$ & \cellcolor{bronze!60}$\num{3.24e-03}$ & \cellcolor{bronze!60}$\num{5.84e-03}$ & \cellcolor{bronze!60}$\num{2.18e-03}$ & \cellcolor{bronze!60}$\num{9.67e-03}$ & $\num{1.16e-02}$ & \cellcolor{bronze!60}$\num{1.18e-02}$ & \cellcolor{bronze!60}$\num{1.04e-02}$ & \cellcolor{bronze!60}$\num{1.37e-02}$ & \cellcolor{bronze!60}$\num{1.61e-02}$ & $\num{3.11e-02}$ & \cellcolor{bronze!60}$\num{9.80e-03}$ \\ 
\hspace{7pt}2nd Quartile & \cellcolor{bronze!60}$\num{2.39e-03}$ & \cellcolor{bronze!60}$\num{4.22e-03}$ & \cellcolor{bronze!60}$\num{1.02e-02}$ & \cellcolor{bronze!60}$\num{1.87e-02}$ & \cellcolor{bronze!60}$\num{1.02e-02}$ & $\num{2.47e-02}$ & $\num{4.19e-02}$ & \cellcolor{bronze!60}$\num{2.95e-02}$ & \cellcolor{bronze!60}$\num{1.79e-02}$ & \cellcolor{bronze!60}$\num{3.27e-02}$ & \cellcolor{bronze!60}$\num{3.36e-02}$ & $\num{8.47e-02}$ & \cellcolor{bronze!60}$\num{2.59e-02}$ \\ 
\hspace{7pt}3rd Quartile & \cellcolor{bronze!60}$\num{5.92e-03}$ & \cellcolor{bronze!60}$\num{1.23e-02}$ & \cellcolor{bronze!60}$\num{3.00e-02}$ & \cellcolor{bronze!60}$\num{3.37e-02}$ & \cellcolor{bronze!60}$\num{2.48e-02}$ & $\num{5.24e-02}$ & $\num{9.24e-02}$ & $\num{7.81e-02}$ & \cellcolor{bronze!60}$\num{4.65e-02}$ & $\num{7.32e-02}$ & $\num{6.25e-02}$ & $\num{1.59e-01}$ & $\num{5.59e-02}$ \\ 
\hline
\end{tabular}}
    \label{tab:probabilistic_small_n}
\end{table}

\begin{figure}[ht]
    \centering
    \includegraphics[width=\linewidth]{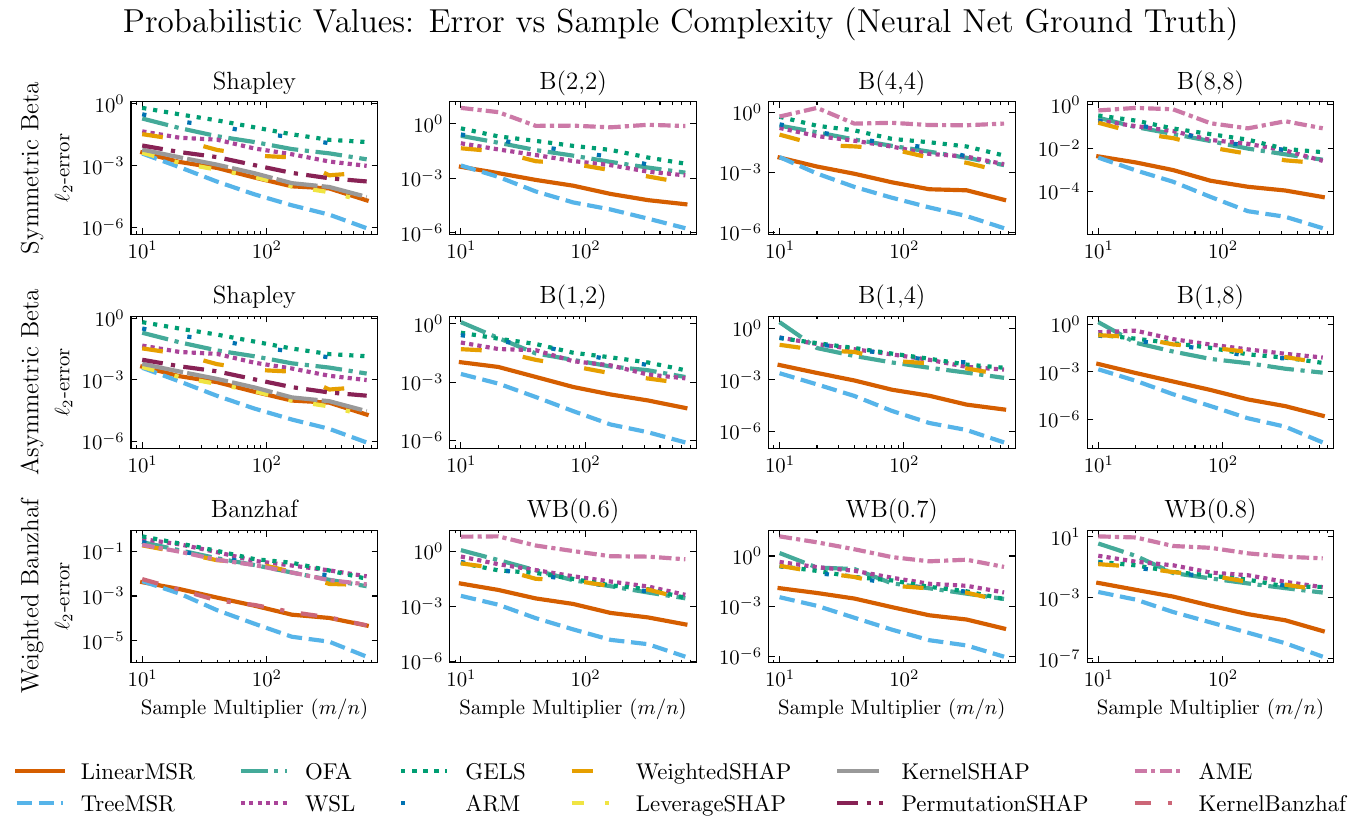}
    \caption{Error between the estimated and true probabilistic values by complexity. Each subplot shows results for a different probabilistic value with the error averaged over all small datasets ($n < 30$). The lines report the mean error over 10 runs. Tree and Linear MSR give the best performance, often by several orders of magnitude especially when the number of samples is large.}
    \label{fig:prob_complexity_small_n}
\end{figure}

\clearpage

\section{Experiments by Noise}\label{appendix:experiments_noise}

In many settings, access to the value function is noisy. For example, $v$ may be the expectation over a distribution that is expensive to exactly compute. Instead, we may estimate the expectation, and hence the values we observe are noisy. In this experiment, we add normally distributed noise to the values passed into each estimator. The plots show the performance of each estimator by the magnitude of this noise.

\begin{figure}[ht]
    \centering
    \includegraphics[width=\linewidth]{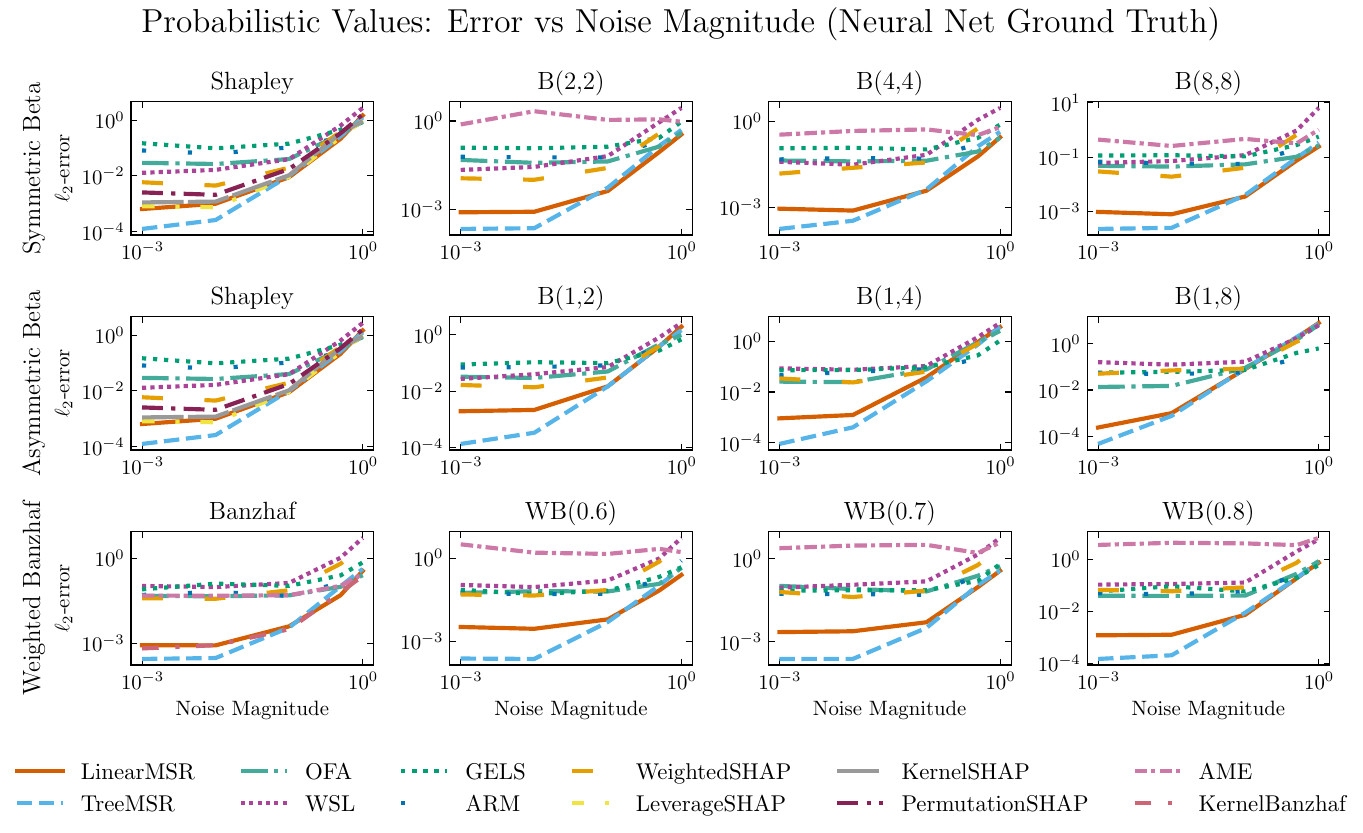}
    \caption{Error between the estimated and true probabilistic values as a function of noise magnitude. Each subplot shows results for a different probabilistic value with the error averaged over all small datasets ($n < 30$). The lines report the mean error over 10 runs. Tree MSR gives the best performance, often by several orders of magnitude especially when the magnitude of the noise is small.}
    \label{fig:prob_complexity_small_n_noise}
\end{figure}

\begin{figure}[ht]
    \centering
    \includegraphics[width=\linewidth]{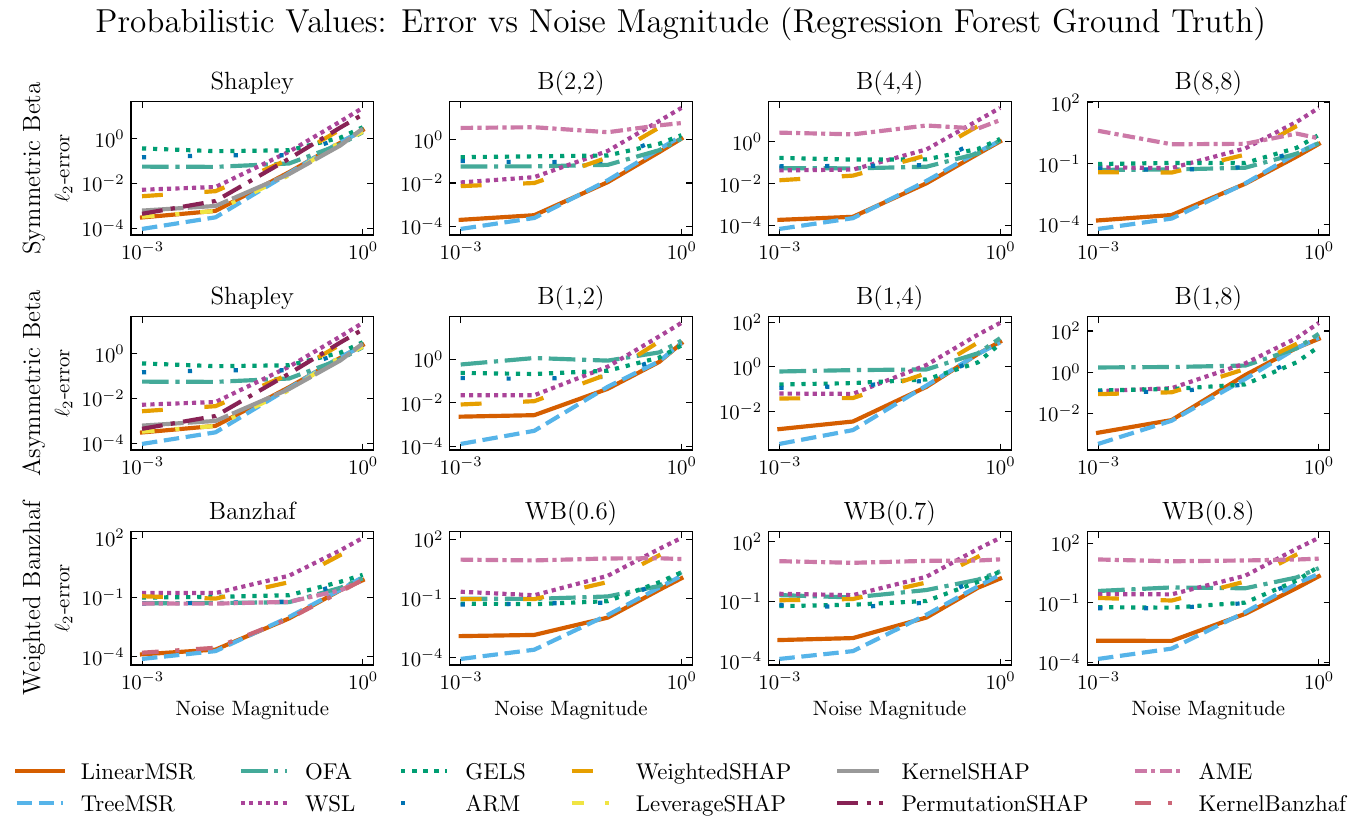}
    \caption{Error between the estimated and true probabilistic values as a function of noise magnitude. Each subplot shows results for a different probabilistic value with the error averaged over all large datasets ($n \geq 30$). The lines report the mean error over 10 runs. Tree MSR gives the best performance, often by several orders of magnitude especially when the magnitude of the noise is small.}
    \label{fig:prob_complexity_big_n_noise}
\end{figure}

\clearpage

\section{Dataset Descriptions}\label{appendix:datasets}

\begin{table}[ht]
\centering
\caption{A summary of the datasets used in our experiments, including source, access method, license, and number of features \(n\).}
\small
\begin{tabular}{l|l|l|l|l}
\toprule
\textbf{Dataset} & $n$ & \textbf{Source / Citation} & \textbf{Access Method} & \textbf{License} \\
\midrule
\textbf{Adult} & 12 & \cite{kohavi1996scaling} & \texttt{shap.datasets} & CC-BY 4.0 \\
\textbf{Forest Fires} & 13 & \cite{cortez2007data} & UCI ML Repo\tablefootnote{\url{https://archive.ics.uci.edu/ml/datasets/forest+fires}} & CC-BY 4.0 \\
\textbf{Real Estate} & 15 & \cite{yeh2009comparisons} & UCI ML Repo\tablefootnote{\url{https://archive.ics.uci.edu/ml/datasets/Real+estate+valuation+data+set}} & CC-BY 4.0 \\
\textbf{Bike Sharing} & 16 & \cite{fanaee2014event} & OpenML\tablefootnote{\url{https://www.openml.org/d/42712}} & Public Domain \\
\textbf{Breast Cancer} & 30 & \cite{street1993nuclear} & \texttt{sklearn.datasets} & CC-BY 4.0 \\
\textbf{Independent} & 60 & \cite{lundberg2017unified} & \texttt{shap.datasets} & MIT \\
\textbf{NHANES} & 79 & \cite{cdc_nhanes} & \texttt{shap.datasets} & Public Domain \\
\textbf{Communities} & 101 & \cite{redmond2002communities} & \texttt{shap.datasets} & CC-BY 4.0 \\
\bottomrule
\end{tabular}
\end{table}

\end{document}